\documentclass[11pt]{article}
\usepackage{fullpage,times,url,bm}
\usepackage[utf8x]{inputenc}
\usepackage[T1]{fontenc}
\usepackage{import}
\usepackage{bbm}
\usepackage[dvipsnames]{xcolor}
\usepackage{hyperref}
\hypersetup{colorlinks=true,citecolor=BlueViolet,linkcolor=BrickRed}
\usepackage{enumitem}
\usepackage{framed}
\usepackage{amsthm,amsmath,amssymb,amsfonts}
\usepackage{pdfpages}
\usepackage{lmodern}
\usepackage{caption}
\usepackage{subcaption}
\usepackage[capitalize]{cleveref}
\usepackage{thm-restate}
\usepackage{comment}
\usepackage{animate}
\usepackage{wrapfig}
\usepackage{microtype}
\usepackage{booktabs}
\usepackage[a4paper,margin=1in]{geometry}
\usepackage[round]{natbib}

\newcommand{\andwith}{\quad\quad}
\usepackage{graphicx}
\graphicspath{ {./figures/} }
\usepackage{float}
\usepackage{relsize}
\usepackage{etoolbox}
\patchcmd{\abstract}{\small}{}{}{}
\usepackage[english]{isodate}

\newcommand{\stam}[1]{}
\newcommand{\secref}[1]{Sec.~\ref{#1}}

\newcommand{\figref}[1]{Fig.~\ref{#1}}
\renewcommand{\eqref}[1]{Eq.~(\ref{#1})}
\newcommand{\lemref}[1]{Lemma~\ref{#1}}

\newcommand{\thmref}[1]{Thm.~\ref{#1}}

\DeclareMathOperator{\rank}{rank}

\DeclareMathOperator{\cl}{cl}

\DeclareMathOperator{\spn}{span}
\newtheorem{theorem}{Theorem}
\newtheorem{assumption}{Assumption}
\newtheorem{definition}{Definition}

\newtheorem{lemma}{Lemma}
\usepackage{dsfont}
\newcommand{\onefunc}{\mathds{1}}
\newcommand{\din}{{d_{\text{in}}}}
\newcommand{\dhid}{{d_{\text{hidden}}}}
\newcommand{\dout}{{d_{\text{out}}}}

\newcommand{\norm}[2][]{{\left\|{#2}\right\|_{#1}}}
\newcommand{\snorm}[1]{\|#1\|} %small norm

\newcommand{\set}[1]{\left\{{#1}\right\}}

\newcommand{\ba}{\mathbf{a}}

\newcommand{\bx}{\mathbf{x}}
\newcommand{\bw}{\mathbf{w}}

\newcommand{\bb}{\mathbf{b}}
\newcommand{\bu}{\mathbf{u}}
\newcommand{\bv}{\mathbf{v}}

\newcommand{\bd}{\mathbf{d}}

\newcommand{\bh}{\mathbf{h}}
\newcommand{\by}{\mathbf{y}}

\newcommand{\bs}{\mathbf{s}}

\newcommand{\btheta}{{\boldsymbol{\theta}}}

\newcommand{\ca}{{\cal A}}

\newcommand{\cd}{{\cal D}}

\newcommand{\cw}{{\cal W}}
\newcommand{\cf}{{\cal F}}

\newcommand{\cs}{{\cal S}}

\newcommand{\calr}{{\cal R}}

\newcommand{\reals}{{\mathbb R}}

\newcommand{\zero}{{\mathbf{0}}}

\title{Implicit Regularization Towards Rank Minimization\\in ReLU Networks}
\author{
    Nadav Timor \andwith Gal Vardi \andwith Ohad Shamir\\
    Weizmann Institute of Science\\
    \texttt{\{nadav.timor,gal.vardi,ohad.shamir\}@weizmann.ac.il}
}
\date{\printdayoff\today}

\begin{document}

\maketitle

\begin{abstract}
    \normalsize
    We study the conjectured relationship between the implicit~regularization in neural~networks, trained with gradient-based methods, and rank~minimization of their weight~matrices. Previously, it was proved that for linear~networks (of depth~$2$ and vector-valued~outputs), gradient~flow~(GF) w.r.t. the square~loss acts as a rank~minimization~heuristic. However, understanding to what extent this generalizes to nonlinear~networks is an open~problem. In this paper, we focus on nonlinear~ReLU~networks, providing several new positive and negative results. On the negative side, we prove (and demonstrate empirically) that, unlike the linear~case, GF on ReLU~networks may no~longer tend to minimize~ranks, in a rather~strong sense (even approximately, for ``most'' datasets of size~$2$). On the positive side, we reveal that ReLU~networks of sufficient~depth are provably biased towards low-rank~solutions in several reasonable settings.
\end{abstract}

\section{Introduction}\label{Introduction}

A central puzzle in the theory~of~deep~learning is how neural~networks generalize even when trained without any explicit~regularization, and when there are far more learnable~parameters than training~examples. In such an underdetermined optimization~problem, there are many global~minima with zero~training~loss, and gradient~descent seems to prefer solutions that generalize well (see~\cite{zhang2017understanding}). Hence, it is believed that gradient~descent induces an {\em implicit~regularization}~(or~\emph{implicit~bias})~\citep{neyshabur2015search,neyshabur2017exploring}, and characterizing this regularization/bias has been a subject of extensive research.

Several works in recent years studied the relationship between the implicit~regularization in \emph{linear}~neural~networks and rank~minimization.
A main focus is on the matrix~factorization~problem, which corresponds to training a depth-2~linear~neural~network with multiple outputs w.r.t. the square~loss, and is considered a well-studied test-bed for studying implicit~regularization in deep~learning. 
\cite{gunasekar2018implicit} initially conjectured that the implicit~regularization in matrix~factorization can be characterized by the nuclear~norm of the corresponding linear~predictor. 
This conjecture was further studied in a string of works (e.g., \cite{belabbas2020implicit,arora2019implicit,razin2020implicit}) and was formally refuted by \cite{li2020towards}.
\cite{razin2020implicit} conjectured that the implicit~regularization in matrix~factorization can be explained by rank~minimization, and also hypothesized that some notion of rank~minimization may be key to explaining generalization in deep~learning. \cite{li2020towards} established evidence that the implicit~regularization in matrix~factorization is a heuristic for rank~minimization.
\cite{razin2021implicit} studied implicit~regularization in tensor~factorization (a generalization of matrix~factorization). They demonstrated, both theoretically and empirically, implicit~bias towards low-rank~tensors. Going beyond factorization~problems,  \cite{ji2018gradient, ji2020directional} showed that in linear~networks of output~dimension~$1$, gradient~flow~(GF) w.r.t. exponentially-tailed~classification~losses converges to networks where the weight~matrix of every layer is of rank~$1$.

However, once we move to nonlinear~neural~networks (which are by far the more common in~practice), things are less~clear. 
Empirically, a series of works studying neural~network~compression (cf. \cite{denton2014exploiting,yu2017compressing,alvarez2017compression,arora2018stronger,tukan2020compressed}) showed that replacing the weight~matrices by low-rank~approximations results in only a small drop in accuracy. This suggests that the weight~matrices in~practice are not too far from being low-rank. However, whether they provably behave this way remains unclear.

In this work we consider fully-connected~nonlinear~networks employing the popular ReLU~activation~function, and study whether GF is biased towards networks where the weight~matrices have low~ranks. On the negative side, we show that already for small~(depth~and~width~$2$)~ReLU~networks, there is~no rank-minimization~bias in a rather strong sense. On the positive side, for deeper~and~possibly~wider~overparameterized~networks, we identify reasonable settings where GF is biased towards low-rank~solutions. In more details, our contributions are as follows:
\begin{itemize}[itemsep=3pt,parsep=3pt]
	\item We begin by considering depth-$2$~width-$2$~ReLU~networks with multiple outputs, trained with the square~loss.  \cite{li2020towards} gave evidence that in \emph{linear} networks with the same architecture, the implicit~bias of GF can be characterized as a heuristic for rank~minimization. In contrast, we show that in ReLU~networks, the situation is quite different: Specifically, we show that GF does~not converge to a low-rank~solution, already for the simple case of datasets of size~$2$, $\{(\bx_1,\by_1),(\bx_2,\by_2)\} \subseteq \mathbb{S}^1 \times \mathbb{S}^1$, whenever the angle between $\bx_1$ and $\bx_2$ is in $(\pi/2,\pi)$ and $\by_1,\by_2$ are linearly~independent. Thus, rank~minimization does~not occur even if we just consider ``most'' datasets of this size. Moreover, we show that with at least constant probability, the solutions that GF converges to are not even close to have low~rank, under any reasonable approximate rank metric. We also demonstrate these results empirically.
	
	\item Next, for ReLU~networks that are overparameterized in~terms of depth and have width~$\ge 2$, we identify interesting settings in which GF is biased towards low~ranks:
	\begin{itemize}[leftmargin=*,itemsep=3pt,parsep=3pt, topsep=3pt, partopsep=3pt]
    	\item First, we consider ReLU~networks trained w.r.t. the square~loss. We show that for sufficiently deep networks, if GF converges to a network that attains zero~loss and minimizes the $\ell_2$~norm of the parameters, then the average~ratio between the spectral and the Frobenius~norms of the weight~matrices is close~to~$1$. Since the squared inverse of this ratio is the \emph{stable~rank} (which is a continuous approximation of the rank, and equals~$1$ if and only if the matrix has rank~$1$), the result implies a bias towards low~ranks.
    	While GF in ReLU~networks w.r.t. the square~loss is not~known to be biased towards solutions that minimize the $\ell_2$~norm, in practice it is common to use explicit $\ell_2$~regularization, which encourages norm~minimization. Thus, our result suggests that GF in deep networks trained with the square~loss and explicit $\ell_2$~regularization encourages rank~minimization.
    	
    	\item Shifting our attention to binary~classification~problems, we consider ReLU~networks trained with exponentially-tailed~classification~losses. By \cite{lyu2019gradient}, GF in such networks is biased towards networks that maximize the $\ell_2$ margin. We show that for sufficiently deep networks, maximizing the margin implies rank~minimization, where the rank is measured by the ratio between the norms as in the former case.
	\end{itemize}
\end{itemize}

\subsection*{Additional Related Work}

The implicit~regularization in matrix~factorization and linear~neural~networks with the square~loss was extensively studied, as a first step toward understanding implicit~regularization in more complex models (see, e.g., \cite{gunasekar2018implicit,razin2020implicit,arora2019implicit,belabbas2020implicit,eftekhari2020implicit,li2018algorithmic,ma2018implicit,woodworth2020kernel,gidel2019implicit,li2020towards,yun2020unifying,azulay2021implicit,razin2021implicit}). As we already discussed, some of these works showed bias toward low~ranks.

The implicit~regularization in nonlinear~neural~networks with the square~loss was studies in several works.
\cite{oymak2019overparameterized} showed that under some assumptions, gradient~descent in certain nonlinear models is guaranteed to converge to a zero-loss~solution with a bounded $\ell_2$~norm.
\cite{williams2019gradient} and \cite{jin2020implicit} studied the dynamics and implicit~bias of gradient~descent in wide~depth-$2$~ReLU~networks with input~dimension~$1$. 
\cite{vardi2021implicit} and \cite{azulay2021implicit} studied the implicit~regularization in single-neuron~networks.
In particular, \cite{vardi2021implicit} showed that in single-neuron~networks and single-hidden-neuron~networks with the ReLU~activation, the implicit~regularization cannot be expressed by any non-trivial~regularization~function. Namely, there is~no non-trivial~regularization~function~$\calr(\btheta)$, where $\btheta$ are the parameters of the model, such that if GF with the square~loss converges to a global~minimum, then it is a global~minimum that minimizes~$\calr$. However, this negative result does~not imply that GF is~not implicitly~biased towards low-rank~solutions, for two reasons. First, bias toward low~ranks would not have implications in the case of networks of width~$1$ that these authors studied, and hence it would not contradict their negative result. Second, their result rules~out the existence of a non-trivial~regularization~function which expresses the implicit~bias for all possible datasets and initializations, but it does~not rule~out the possibility that GF acts as a heuristic for rank~minimization, in the sense that it minimizes the ranks for ``most'' datasets and initializations.

The implicit~bias of neural~networks in classification tasks was also widely studied in recent years.
\cite{soudry2018implicit} showed that gradient~descent on linearly-separable~binary~classification~problems with exponentially-tailed~losses, converges to the maximum~$\ell_2$-margin~direction. This analysis was extended to other loss~functions, tighter convergence~rates, non-separable~data, and variants of gradient-based~optimization~algorithms \citep{nacson2019convergence,ji2018risk,ji2020gradient,gunasekar2018characterizing,shamir2021gradient,ji2021characterizing}.
\cite{lyu2019gradient} and \cite{ji2020directional} showed that GF on homogeneous~neural~networks, with exponentially-tailed~losses, converges~in~direction to a KKT~point of the maximum-margin~problem in the parameter~space. Similar results under stronger assumptions were previously obtained in \cite{nacson2019lexicographic,gunasekar2018bimplicit}. 
\cite{vardi2021margin} studied in which settings this KKT~point is guaranteed to be a global/local~optimum of the maximum-margin~problem.
The implicit~bias in fully-connected~linear~networks was studied by \cite{ji2020directional,ji2018gradient,gunasekar2018bimplicit}. As already mentioned, these results imply that GF minimizes the ranks of the weight~matrices in linear~fully-connected~networks.
The implicit~bias in diagonal and convolutional~linear~networks was studied in \cite{gunasekar2018bimplicit,moroshko2020implicit,yun2020unifying}. 
The implicit~bias in infinitely-wide~two-layer~homogeneous~neural~networks was studied in \cite{chizat2020implicit}. 

{\bf Organization.} In~\secref{Preliminaries} we provide necessary notations and definitions. In~\secref{sec:negative results} we state our negative results for depth-$2$~networks. In~\secref{sec:positive results ell2} and~\ref{sec:positive results exp} we state our positive results for deep~ReLU~networks. In~\secref{sec:ideas} we describe the ideas for the proofs of the main theorems, with all formal proofs deferred to the appendix.

\section{Preliminaries}\label{Preliminaries}

\paragraph{Notations.}

We use boldface letters to denote vectors.
For~$\bx \in \mathbb{R}^d$ we denote by~$\|\bx\|$ the Euclidean~norm. For a matrix~$X$ we denote by~$\norm{X}_F$ the Frobenius~norm and by~$\norm{X}_\sigma$ the spectral~norm.
We denote~$\reals_+ := \{x \in \reals: x \geq 0\}$.
For an integer~$d \geq 1$ we denote~$[d]:=\{1,\ldots,d\}$. 
The angle between a pair of vectors~\(\bu_1, \bu_2 \in \mathbb R^d\) is~\hbox{$\measuredangle(\bu_1, \bu_2) := \arccos \left( \frac {\bu_1^\top \bu_2} {\|\bu_1\| \cdot \|\bu_2\|} \right) \in [ 0, \pi]$}. 
The unit~\(d\)-sphere is~\(\mathbb{S}^d := \{\bu \in \mathbb{R}^{d+1} \mid \|\bu\| = 1\} \).
An open~\(d\)-ball that is centered at the origin is denoted by~\(B_d(\epsilon) := \{ \bu \in \mathbb{R}^d \mid \|\bu\| < \epsilon \}\) for some~\(\epsilon \in \mathbb{R}\). The closure of a set~\(A \in \mathbb{R}^d\), denoted as~\(\cl A\), is the intersection of all closed~sets containing~\(A\). The boundary of~\(A\) is~\(\partial A := \left( \cl A \right) \cap \left( \cl \left( \mathbb{R}^d \setminus A \right) \right)\).

\paragraph{Neural networks.}

A {\em fully-connected~neural~network} $N_\btheta$ of depth~$k \geq 2$ is parameterized by a collection~$\btheta := [W^{(l)}]_{l=1}^k$ of weight~matrices, such that for every layer~$l \in [k]$ we have \hbox{$W^{(l)} \in \reals^{d_l \times d_{l-1}}$}. Thus, $d_l$ denotes the number of neurons in the $l$-th~layer, i.e., the \emph{width~of~the~layer}. We denote by~$\din:=d_0$, $\dout:=d_k$ the input and output~dimensions. 
The neurons in layers~$[k-1]$ are called \emph{hidden~neurons}.
A fully-connected~network computes a function~$N_\btheta: \reals^\din \to \reals^\dout$ defined recursively as follows. For an input~$\bx \in \reals^\din$ we set $\bh'_0 := \bx$, and define for every~$j \in [k-1]$ the input to the $j$-th~layer as~$\bh_j := W^{(j)} \bh'_{j-1}$, and the output of the $j$-th~layer as $\bh'_j := \sigma(\bh_j)$, where $\sigma:\reals \to \reals$ is an activation~function that acts coordinate-wise on vectors. 
In this work we focus on the ReLU~activation~$\sigma(z) = \max\{0,z\}$.
Finally, we define~$N_\btheta(\bx) := W^{(k)} \bh'_{k-1}$. Thus, there is~no activation in the output~neurons. 
The \emph{width~of~the~network}~$N_\btheta$ is the maximal~width of its layers, i.e.,~$\max_{l \in [k]}d_l$.
We sometimes apply the activation~function $\sigma$ also on matrices, in which case it acts entry-wise.
The parameters~$\btheta$ of the neural~network are given by a collection of matrices, but we often view~$\btheta$ as the vector obtained by concatenating the matrices in the collection. Thus, $\norm{\btheta}$~denotes the $\ell_2$~norm of the vector~$\btheta$.

We often consider depth-$2$~networks. For matrices~$W \in \reals^{d_1 \times d_0}$ and~$V \in \reals^{d_2 \times d_1}$ we denote by~$N_{W,V}$ the depth-$2$~ReLU~network where~$W^{(1)}=W$ and~$W^{(2)}=V$. We denote the $\bw_1^\top,\ldots,\bw_{d_1}^\top$ the rows of~$W$, namely, the incoming~weight~vectors to the neurons in the hidden~layer, and by~$\bv_1,\ldots,\bv_{d_1}$ the columns of~$V$, namely, the outgoing~weight~vectors from the neurons in the hidden~layer.

Let $\bx_1,\ldots,\bx_n$ be $n$ inputs, and let $X \in \reals^{\din \times n}$ be a matrix whose columns are $\bx_1,\ldots,\bx_n$. We denote by $N_\btheta(X) \in \reals^{\dout \times n}$ the matrix whose $i$-th column is $N_\btheta(\bx_i)$.

\paragraph{Optimization problem and gradient flow (GF).}

Let $S = \{(\bx_i,\by_i)\}_{i=1}^n \subseteq \reals^\din \times \reals^\dout$ be a training dataset. 
We often represent the dataset by matrices $(X,Y) \in \reals^{\din \times n} \times \reals^{\dout \times n}$.
For a neural~network~$N_\btheta$ we consider empirical-loss~minimization w.r.t. the square~loss. 
Thus, the objective is given by:
\begin{align} \label{eq:objective}
	 L_{X,Y}(\btheta) 
	 := \frac{1}{2} \sum_{i=1}^n \norm{N_\btheta(\bx_i) - \by_i}^2
	 = \frac{1}{2} \norm{N_\btheta(X) - Y}_F^2~.
\end{align}
We assume that the data is realizable, that is, $\min_\btheta L(\btheta)=0$.
Moreover, we focus on settings where the network is {\em overparameterized}, in the sense that $L$ has multiple (or even infinitely many) global minima.

We consider \emph{gradient~flow~(GF)} on the objective given in \eqref{eq:objective}. This setting captures the behavior of gradient~descent with an infinitesimally small step~size. Let~$\btheta(t)$ be the trajectory of GF. Starting from an initial point~$\btheta(0)$, the dynamics of~$\btheta(t)$ is given by the differential equation~$\frac{d \btheta(t)}{dt} = -\nabla L_{X,Y}(\btheta(t))$.
Note that the ReLU~function is~not differentiable at~$0$. Practical implementations of gradient~methods define the derivative~$\sigma'(0)$ to be some constant in~$[0,1]$. 
In this work we assume for convenience that~$\sigma'(0)=0$.
We say that GF~\emph{converges} if $\lim_{t \to \infty}\btheta(t)$~exists. In this case, we denote~$\btheta(\infty) := \lim_{t \to \infty}\btheta(t)$.

\section{Gradient~flow does~not even approximately minimize~ranks} \label{sec:negative results}

In this section we consider rank~minimization in depth-$2$~networks~$N_{W,V}$ trained with the square~loss. We show that even for the simple case of size-$2$~datasets, under mild assumptions, GF does~not converge to a minimum-rank~solution even approximately. 

In what follows, we consider ReLU~networks with vector-valued~outputs, since for linear~networks with the same architecture it was shown that GF can be viewed as a heuristic for rank~minimization (cf. \cite{li2020towards,razin2020implicit}). Specifically, let $(X, Y) \in \mathbb{R}^{2 \times 2} \times \mathbb{R}^{2 \times 2}$ be a training dataset, and let $W,V\in \reals^{2 \times 2}$ be weight~matrices such that $N_{W,V}(X)=V\sigma(WX)$ is a zero-loss~solution. Note that if $\rank(Y)=2$ then we must have $\rank(V)=2$: Indeed, by definition of $N_{W,V}$, we necessarily have $\rank(Y)=\rank(N_{W,V}(X)) \leq \rank(V)$. Therefore, to understand rank~minimization in this simple setting, we consider the rank of $W$ in a zero-loss~solution. Trivially, $\rank(W)\leq 2$, so $W$ can be considered low-rank only if $\rank(W)\leq 1$. 

To make the setting non-trivial, we need to show that such low-rank~zero-loss~solutions exist at all. The following theorem shows that this is true for almost all size-$2$ datasets:
\begin{theorem}\label{T1}
    Given any labeled dataset $(X, Y) \in \mathbb{R}^{2 \times 2} \times \mathbb{R}^{2 \times 2}$ of two inputs $\bx_1, \bx_2 \in \mathbb{R}^2$ with a strictly positive angle between them, i.e., $\measuredangle(\bx_1, \bx_2) > 0$, there exists a zero-loss~solution $N_{W,V}$ with $W,V \in \reals^{2 \times 2}$, such that $\rank(W)=1$.
\end{theorem}

The theorem follows by constructing a network where the weight~vectors of the neurons in the first~layer have opposite directions (and hence the weight~matrix is of rank $1$), such that each neuron is active for exactly one input. Then, it is possible to show that for an appropriate choice of the weights in the second layer the network achieves zero~loss.
See Appendix~\ref{app:T1} for the formal proof.

\thmref{T1} implies that zero-loss~solutions of rank~$1$ exist. However, we now show that GF does~not converge to such solutions. 
We prove this result under the following assumptions:

\begin{assumption}\label{assumption:ys_linearly_independent}
    The two target~vectors~$\by_1, \by_2$ are on the unit~sphere~$\mathbb{S}^1$ and are linearly~independent.
\end{assumption}

\begin{assumption}\label{assumption:xs_bounded_angle}
    The two inputs~$\bx_1, \bx_2$ are on the unit~sphere~$\mathbb{S}^1$, and satisfy $\frac{\pi}{2} < \measuredangle(\bx_1, \bx_2) < \pi$.
\end{assumption}

The assumptions that $\bx_i,\by_i$ are of unit norm are mostly for technical convenience, and we believe that they are not essential.

Then, we have:

\begin{theorem}\label{T2}
    Let $(X,Y) \in \mathbb{R}^{2 \times 2} \times \mathbb{R}^{2 \times 2}$ be a labeled dataset that satisfies Assumptions~\ref{assumption:ys_linearly_independent} and~\ref{assumption:xs_bounded_angle}. 
    Consider GF w.r.t. the loss~function~$L_{X,Y}(W,V)$ from~\eqref{eq:objective}. Suppose that $W,V \in \reals^{2 \times 2}$ are initialized such that
    \[
        \norm{\bw_i(0)} < \min \left\{ \frac{1}{2}, \frac{\sqrt{3}}{2} \cos{\frac{\measuredangle(\bx_1, \bx_2)}{2}} \right\}
    \]
    and $\norm{\bv_i(0)} < \frac{1}{2}$ for all $i \in \{1,2\}$. If GF converges to a zero-loss~solution~$N_{W(\infty),V(\infty)}$, then $\rank(W(\infty))=2$.
\end{theorem}

By the above theorem, GF does~not minimize the rank even in a very simple setting where the dataset contains two inputs with angle larger than $\pi/2$ (as long as the initialization point is sufficiently close~to~$0$). In particular, if the dataset is drawn from the uniform distribution on the sphere then this condition holds with probability $1/2$. 

While \thmref{T2} shows that GF does~not minimize the rank, it does~not rule~out the possibility that it converges to a solution which is close to a low-rank~solution. There are many ways to define such closeness, such as the ratio of the Frobenius and spectral~norms, the Frobenius~distance from a low-rank~solution, or the exponential of the entropy of the singular~values (cf. \cite{rudelson2007sampling,sanyal2019stable,razin2020implicit,roy2007effective}). However, for $2\times 2$~matrices they all boil down to either having the two rows of the matrix being nearly aligned, or having at~least one of them very small (at~least compared to the other). In the following theorem, we show that under the assumptions stated above, for any fixed dataset, with at~least constant probability, GF converges to a zero-loss~solution, where the two row vectors are bounded away from~$0$, the ratio of their norms are bounded, and the angle between them is bounded~away from~$0$ and from~$\pi$ (all by explicit constants that depend just on the dataset and are large in general). Thus, with at least constant probability, GF does~not minimize any reasonable approximate notion of rank.

\begin{theorem} \label{T3}
    Let $(X,Y) \in \mathbb{R}^{2 \times 2} \times \mathbb{R}^{2 \times 2}$ be a labeled dataset that satisfies Assumptions~\ref{assumption:ys_linearly_independent} and~\ref{assumption:xs_bounded_angle}.
    Consider GF w.r.t. the loss~function $L_{X,Y}(W,V)$ from \eqref{eq:objective}. Suppose that $W,V \in \reals^{2 \times 2}$ are initialized such that for all $i \in \{1,2\}$ we have $\bv_i(0) = \zero$, and $\bw_i(0)$ is drawn from a spherically~symmetric distribution with 
    \begin{align*}
        \norm{\bw_i(0)} 
        \leq \frac{\sqrt{3}}{2} \min \left\{ 
        \sin\left(\frac{\pi - \measuredangle(\bx_1,\bx_2)}{4} \right), 
        \sin \left(\measuredangle(\bx_1,\bx_2) - \frac{\pi}{2}\right)
        \right\} ~.
    \end{align*}
    Let $E$ be the event that GF converges to a zero-loss~solution~$N_{W(\infty),V(\infty)}$ such that 
    \begin{itemize}
        \item[(i)] $\measuredangle \left( \bw_1(\infty), \bw_2(\infty) \right) \in \left[ \frac{\pi}{2} - \left(\measuredangle(\bx_1, \bx_2) - \frac{\pi}{2} \right), \frac{3\pi}{4} + \frac{\measuredangle(\bx_1, \bx_2) - \pi/2}{2} \right]$,
        \item[(ii)] $\norm{\bw_i(\infty)} \in \left( \frac{\sqrt{3}}{2}, \sqrt{ \frac{1}{4} + \frac{4}{3 \left(\sin \measuredangle(\bx_1, \bx_2) \right)^2 } } \right)$ for all $i \in \set{1, 2}$.
    \end{itemize}
    Then, $\Pr \left[E \right] \ge 2 \cdot \left(\frac{\measuredangle(\bx_1,\bx_2)}{2\pi} \right)^2$.
\end{theorem}

We note that in \thmref{T3} the weights in the second layer are initialized to zero, while in \thmref{T2} the assumption on the initialization is weaker. This difference is for technical convenience, and we believe that \thmref{T3} should hold also under weaker assumptions on the initialization, as the next empirical result demonstrates.

\subsection{An empirical result}

Our theorems imply that for standard initialization schemes, GF will not converge close~to low-rank~solutions, with some positive~probability. We now present a simple experiment that corroborates this and suggests that, furthermore, this holds with~high~probability.

Specifically, we trained ReLU~networks in the same setup as in the previous section (w.r.t. two $2\times 2$~weight~matrices~$W^{(1)},W^{(2)}$) on the two data~points~$\{(\bx_i,\by_i)\}_{i=1}^{2}$ where~$\by_1,\by_2$ are the standard~basis vectors in~$\reals^2$, and~$\bx_1,\bx_2$ are~$(1,0.99)$ and~$(-1,0.99)$ normalized to have unit~norm. At initialization, every row of~$W^{(1)}$ and every column of~$W^{(2)}$ is sampled uniformly~at~random from the sphere of radius~$10^{-4}$ around the origin. To simulate GF, we performed~$3\cdot 10^{6}$ epochs of full-batch~gradient~descent of step~size~$10^{-4}$, w.r.t. the square~loss. Of~$288$~repeats of this experiment, $79$~converged to negligible~loss~(defined~as~$<10^{-4}$). In 
\figref{fig:empirical-high-f2s}, we plot a histogram of the \emph{stable~(numerical)~ranks} of the resulting weight~matrices, i.e. the ratio~${\norm{W^{(\ell)}}^2_F} / {\norm{W^{(\ell)}}^2_\sigma}$ of layer~$\ell \in [2]$. The figure clearly suggests that whenever convergence to zero~loss occurs, the solutions are all of rank~$2$, and none are even close~to being low-rank (in~terms of the stable~rank).

\begin{figure}[t]
    \centering
    \frame{\includegraphics[scale=.6]{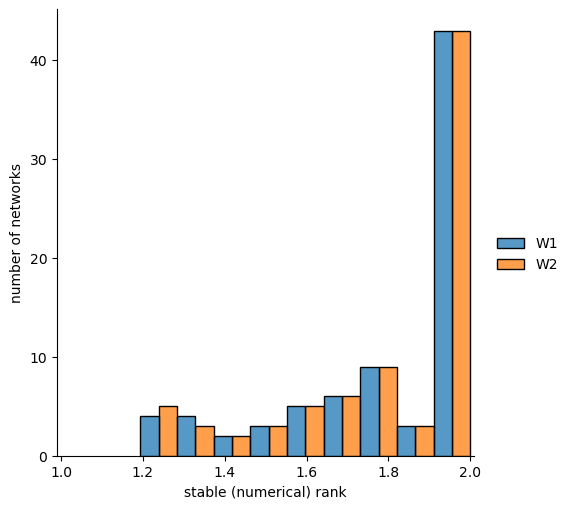}}
  \caption{A histogram of the stable (numerical) ranks at convergence. In all runs, we converge to networks with stable~ranks which seem bounded away from $1$. Namely, gradient~descent does~not even approximately minimize the ranks.}
  \label{fig:empirical-high-f2s}
\end{figure}

\section{Rank~minimization in deep~networks with small~$\ell_2$~norm} \label{sec:positive results ell2}

When training neural~networks with gradient~descent, it is common to use explicit $\ell_2$~regularization on the parameters. In this case, gradient~descent is biased towards solutions that minimize the $\ell_2$~norm of the parameters. 
We now show that in deep~overparameterized~ReLU~networks, if GF converges to a zero-loss~solution that minimizes the $\ell_2$~norm, then the ratios between the Frobenius and the spectral~norms in the weight~matrices tend to be small (we use here the ratio between these norms as a continuous surrogate for the exact rank, as discussed in the previous section). Formally, we have the following:

\begin{theorem}\label{thm:positive-square-loss}
    Let~$\left\{ \left( \bx_i, y_i \right) \right\}_{i=1}^n \subseteq \mathbb{R}^{\din} \times \reals_+$ be a dataset, and assume that there~is~$i \in [n]$ with~$\norm{\bx_i} \leq 1$ and $y_i \geq 1$. Assume that there~is a fully-connected~neural~network~$N$ of width~$m \geq 2$ and depth~$k \geq 2$, such that for~all~$i \in [n]$ we have~$N(\bx_i) = y_i$, and the weight~matrices~$W_1,\ldots,W_k$ of $N$ satisfy $\norm{W_i}_F \leq B$ for some~$B>0$.
    Let~$N_\btheta$ be a fully-connected~neural~network of width~$m' \geq m$ and depth~$k' > k$ parameterized by~$\btheta$. Let~$\btheta^* = \left[W_1^*,\ldots,W_{k'}^*\right]$ be a global~optimum of the following problem:
    \begin{equation} \label{eq:optimization problem square loss}
        \min_\btheta \norm{\btheta} \;\;\;\; \text{s.t. } \;\;\; \forall i \in [n] \;\; N_\btheta(\bx_i)=y_i~.   
    \end{equation}
    Then, 
    \begin{equation}  \label{eq:average ratio square}
        \frac{1}{k'}\sum_{i=1}^{k'} \frac{\norm{W_i^*}_\sigma}{\norm{W_i^*}_F} 
        \geq \left( \frac{1}{B} \right)^{\frac{k}{k'}}~.
    \end{equation}
    Equivalently, we have the following upper~bound on the harmonic~mean of the ratios~$\frac{\norm{W_i^*}_F}{\norm{W_i^*}_\sigma}$:
    \begin{equation}  \label{eq:harmonic ratio square}
        \frac{k'}{\mathlarger{\sum}_{i=1}^{k'} \left( \frac{\norm{W_i^*}_F}{\norm{W_i^*}_\sigma} \right)^{-1} }
        \leq 
        B^{\frac{k}{k'}}~.
    \end{equation}
\end{theorem}

By the above theorem if $k'$ is much larger than $k$, then the average~ratio between the spectral and the Frobenius~norms~(\eqref{eq:average ratio square}) is at~least roughly~$1$. Likewise, the harmonic~mean of the ratio between the Frobenius and the spectral~norms~(\eqref{eq:harmonic ratio square}), namely, the square~root of the stable~rank, is at~most roughly~$1$. Noting that both these ratios equal~$1$ if~and~only~if the matrix is of rank~$1$, we see that there is a bias towards low-rank~solutions as the depth~$k'$ of the trained~network increases. 
Note that the result does~not depend on the width of the networks. Thus, even if the width~$m'$ is large, the average~ratio is close~to~$1$.
Also, note that the network~$N$ of depth~$k$ in the theorem might have high~ranks (e.g., rank~$m$ for~each weight~matrix), but once we consider networks of a large~depth~$k'$ then the dataset becomes realizable by a network of small average~rank, and GF converges to such a network.

\section{Rank~minimization in deep~networks with exponentially-tailed losses} \label{sec:positive results exp}

In this section, we turn to consider GF in classification tasks with exponentially-tailed losses, namely, the exponential loss or the logistic~loss. 

Let us first formally define the setting.
We consider neural~networks of output~dimension~$1$,~i.e.,~$\dout=1$. 
Let~$S = \{(\bx_i,y_i)\}_{i=1}^n \subseteq \reals^d \times \{-1,1\}$ be a binary~classification training~dataset. 
Let~$X \in \reals^{\din \times n}$ and $\by \in \reals^n$ be the data~matrix and labels that correspond to~$S$.
Let $N_\btheta$ be a neural~network parameterized by~$\btheta$.
For a loss~function~$\ell:\reals \to \reals$, the empirical~loss of~$N_\btheta$ on the dataset~$S$ is 
\begin{equation}
\label{eq:objective classification}
	L_{X,\by}(\btheta) := \sum_{i=1}^n \ell\left(y_i N_\btheta(\bx_i)\right)~.
\end{equation} 
We focus on the exponential~loss~$\ell(q) = e^{-q}$ and the logistic~loss~$\ell(q) = \log(1+e^{-q})$.
We say that the~dataset~is~\emph{correctly~classified}~by~the~network~$N_\btheta$ if for~all~$i \in [n]$ we have~$y_i N_\btheta(\bx_i) > 0$. 
We consider GF on the objective given in \eqref{eq:objective classification}. We say that a network~$N_\btheta$~is~\emph{homogeneous} if there~exists~$M>0$ such that for~every~$\alpha>0$ and~$\btheta,\bx$ we have~$N_{\alpha \btheta}(\bx) = \alpha^M N_\btheta(\bx)$. Note that fully-connected~ReLU~networks are homogeneous.
We say that a~trajectory~$\btheta(t)$~of~GF {\em converges~in~direction}~to~$\tilde{\btheta}$ if 
\[
    \lim_{t \to \infty}\frac{\btheta(t)}{\norm{\btheta(t)}} = \frac{\tilde{\btheta}}{\snorm{\tilde{\btheta}}}~.
\]
The following well-known result characterizes the implicit~bias in homogeneous~neural~networks trained with the logistic or the exponential~loss:

\begin{lemma}[Paraphrased from \cite{lyu2019gradient} and \cite{ji2020directional}] \label{lemma:directional-convergence}
    Let~$N_{\btheta}$ be a homogeneous~ReLU~neural~network. 
    Consider minimizing the average of either the exponential or the logistic~loss over a binary~classification dataset using~GF.
    Suppose that the average~loss converges to zero as~$t \to \infty$.
    Then, GF converges~in~direction to a first~order stationary~point (KKT~point) of the following maximum~margin~problem in parameter~space:
\begin{equation}
\label{eq:optimization problem}
	\min_\btheta \frac{1}{2} \norm{\btheta}^2 \;\;\;\; \text{s.t. } \;\;\; \forall i \in [n] \;\; y_i N_\btheta(\bx_i) \geq 1~.
\end{equation}
\end{lemma}

The above lemma suggests that GF tends to converge~in~direction to a network with margin~$1$ and small~$\ell_2$~norm. 
In the following theorem we show that in deep~overparameterized~ReLU~networks, if GF converges~in~direction to an optimal~solution of Problem~\ref{eq:optimization problem} (from the above lemma) then the ratios between the Frobenius and the spectral~norms in the weight~matrices tend to be small. Formally, we have the following:

\begin{theorem}\label{thm:positive-exp}
    Let~$\left\{ \left( \bx_i, y_i \right) \right\}_{i=1}^n \subseteq \mathbb{R}^{\din} \times \set{-1,1}$ be a binary~classification dataset, and assume that there~is~$i \in [n]$ with $\norm{\bx_i} \leq 1$. Assume that there~is a fully-connected~neural~network~$N$ of width~$m \geq 2$ and depth~$k \geq 2$, such that for~all~$i \in [n]$ we have~$y_i N(\bx_i) \geq 1$, and the weight~matrices~$W_1,\ldots,W_k$ of~$N$ satisfy~$\norm{W_i}_F \leq B$ for some~$B>0$.
    Let~$N_\btheta$ be a fully-connected~neural~network of width~$m' \geq m$ and depth~$k' > k$ parameterized by~$\btheta$. Let~$\btheta^* = \left[W_1^*,\ldots,W_{k'}^*\right]$ be a global~optimum of Problem~\ref{eq:optimization problem}. Namely, $\btheta^*$~parameterizes a minimum-norm~fully-connected~network of width~$m'$ and depth~$k'$ that labels the dataset correctly with margin~$1$.
    Then, we have
    \begin{equation} \label{eq:average ratio exp}
        \frac{1}{k'} \sum_{i=1}^{k'} \frac{\norm{W_i^*}_\sigma}{\norm{W_i^*}_F} \geq \frac{1}{\sqrt{2}} \cdot \left( \frac{\sqrt{2}}{B} \right)^{\frac{k}{k'}} \cdot \sqrt{\frac{k'}{k'+1}}~.
    \end{equation}
    Equivalently, we have the following upper~bound on the harmonic~mean of the ratios~$\frac{\norm{W_i^*}_F}{\norm{W_i^*}_\sigma}$:
    \begin{equation} \label{eq:harmonic ratio exp}
        \frac{k'}{\mathlarger{\sum}_{i=1}^{k'} \left( \frac{\norm{W_i^*}_F}{\norm{W_i^*}_\sigma} \right)^{-1}}
        \leq \sqrt{2} \cdot \left( \frac{B}{\sqrt{2}} \right)^{\frac{k}{k'}} \cdot \sqrt{\frac{k'+1}{k'}}~.
    \end{equation}
\end{theorem}

By the above theorem, if~$k'$ is much~larger than~$k$, then the average~ratio between the spectral and the Frobenius~norms~(\eqref{eq:average ratio exp}) is at~least roughly~$1/\sqrt{2}$. 
Likewise, the harmonic~mean of the ratio between the Frobenius and the spectral~norms~(\eqref{eq:harmonic ratio exp}), i.e., the square~root of the stable~rank, is at~most roughly~$\sqrt{2}$.
Note that the result does~not depend on the width of the networks. Thus, it holds even if the width~$m'$ is very large.
Similarly to the case of \thmref{thm:positive-square-loss}, we note that the network~$N$ of depth~$k$ might have high~ranks (e.g., rank~$m$ for~each weight~matrix), but once we consider networks~of~a~large~depth~$k'$, then the dataset becomes realizable by a network of small~average~rank, and GF converges to such a network.

The combination of the above result with \lemref{lemma:directional-convergence} suggests that, in overparameterized~deep~fully-connected~networks, GF tends to converge~in~direction to neural~networks with low~ranks. Note that we consider the exponential and the logistic~losses, and hence if the loss tends to zero as~$t \to \infty$, then we have~$\norm{\btheta(t)} \to \infty$. To conclude, in our case, the parameters tend to have an infinite~norm and to converge~in~direction to a low-rank~solution.
Moreover, note that the ratio between the spectral and the Frobenius~norms is invariant~to~scaling, and hence it suggests that after a sufficiently~long~time, GF tends to reach a network with low~ranks. 

\section{Proof ideas} \label{sec:ideas}

In this section we describe the main ideas for the proofs of Theorems~\ref{T2}, \ref{T3}, \ref{thm:positive-square-loss} and~\ref{thm:positive-exp}. The full proofs are given in the appendix.

\subsection{Theorem~\ref{T2}}

We define the following regions (see  \figref{fig:D_and_Ss_regions}):
 \begin{align*}
        &\mathcal{D} := \{ \bw \in \mathbb{R}^2 \mid \forall i \in \{1,2\}, \sigma(\bw^\top \bx_i) \le 0\}~, \\
        &\mathcal{S} := \{ \bw \in \mathbb{R}^2 \mid \forall i \in \{1,2\}, \sigma(\bw^\top \bx_i) > 0 \}~, \\
        &\mathcal{S}_1 := \{ \bw \in \mathbb{R}^2  \mid \sigma(\bw^\top \bx_1) > 0, \sigma(\bw^\top \bx_2) \leq 0\}~, \\
        &\mathcal{S}_2 := \{ \bw \in \mathbb{R}^2  \mid \sigma(\bw^\top \bx_2) > 0, \sigma(\bw^\top \bx_1) \leq 0\}~. 
    \end{align*}

 \begin{figure}[t]
      \centering
      \frame{\includegraphics[scale=.35]{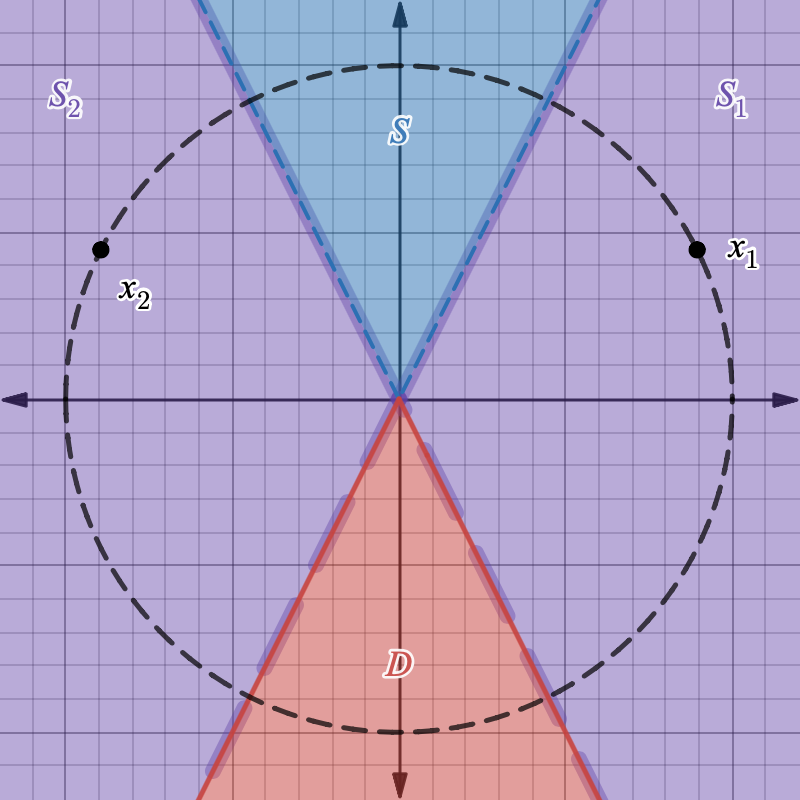}}
      \caption{Regions \(\mathcal{D}\), \(\mathcal{S}_1\), \(\mathcal{S}_2\), \(\mathcal{S}\).}
      \label{fig:D_and_Ss_regions}
\end{figure}
Intuitively,~$\mathcal{D}$~defines the ``dead''~region where the relevant neuron will output~$0$ on both~$\bx_1,\bx_2$; $\mathcal{S}$~is the ``active''~region where the relevant neuron will output a positive~output on both~$\bx_1,\bx_2$; and $\mathcal{S}_1,\mathcal{S}_2$~are the ``partially~active''~regions, where the relevant neuron will output a positive~output on one point, and~$0$ on the other. 

Assume towards~contradiction that GF converges to some zero-loss~network~$N_{W(\infty),V(\infty)}$ with~$\rank(W(\infty))<2$. 
Since $N_{W(\infty),V(\infty)}$ attains zero~loss, then $Y = V(\infty) \sigma \left( W(\infty) X \right)$, and hence 
\begin{align} \label{eq:rank WX}
	2 
	= \rank(Y) 
	= \rank \left( V(\infty) \sigma \left(W(\infty) X \right) \right) 
	\leq \rank \left( \sigma \left( W(\infty) X \right) \right)~.
\end{align} 
Therefore, the weight~vectors~$\bw_1(\infty)$ and~$\bw_2(\infty)$ are~not in the region~$\cd$. Indeed, if $\bw_1(\infty)$ or~$\bw_2(\infty)$ are in~$\cd$, then at~least one of the rows of~$\sigma(W(\infty) X)$ is~zero, in~contradiction to \eqref{eq:rank WX}.
In~particular, it implies that~$\bw_1(\infty)$ and~$\bw_2(\infty)$ are non-zero. Since by our assumption we have~$\rank(W(\infty))<2$, then we conclude that~$\rank(W(\infty))=1$. We denote~$\bw_2(\infty) = \alpha \bw_1(\infty)$ where~$\alpha \neq 0$. Note that if~$\alpha > 0$, then~$\sigma(\bw_2(\infty)^\top \bx_j) = \alpha \sigma(\bw_1(\infty)^\top \bx_j)$ for~all~$j \in \{1,2\}$, in~contradiction to \eqref{eq:rank WX}. Thus,~$\alpha < 0$. Since we also have~$\bw_1(\infty),\bw_2(\infty) \not \in \cd$, then one of these weight~vectors is in~$\cs_1 \setminus \partial \cs_1$ and the other is in~$\cs_2 \setminus \partial \cs_2$ (as can be seen from \figref{fig:D_and_Ss_regions}). Assume~w.l.o.g. that~$\bw_1(\infty) \in \cs_1 \setminus \partial \cs_1$ and $\bw_2(\infty) \in \cs_2 \setminus \partial \cs_2$.

By observing the gradients of~$L_{X,Y}$ w.r.t.~$\bw_i$ for~$i \in \{1,2\}$, the following facts follow. First, if~$\bw_i(t) \in \cd$ at some time~$t$, then~$\frac{d}{dt} \bw_i(t) = \zero$, hence $\bw_i$~remains at~$\cd$ indefinitely, in~contradiction to~$\bw_i(\infty) \in \cs_i \setminus \partial \cs_i$. Thus, the trajectory~$\bw_i(t)$ does~not visit~$\cd$. Second, if~$\bw_i(t) \in \cs_i$ at time~$t$, then~$\frac{d}{dt} \bw_i(t) \in \spn\{\bx_i\}$. Since~$\bw_i(\infty) \in \cs_i \setminus \partial \cs_i$, we can consider the last~time~$t'$ that~$\bw_i$ enters~$\cs_i$, which can be either at the initialization (i.e.,~$t'=0$) or when moving from~$\cs$ (i.e.,~$t'>0$). For~all time~$t \geq t'$ we have~$\frac{d}{dt} \bw_i(t) \in \spn\{\bx_i\}$. It allows us to conclude that~$\bw_i(\infty)$ must be in a region~$\ca_i$ which is illustrated in \figref{fig:F_and_A_regions} (by the union of the orange and green regions).

\begin{figure}[t]
    \centering
    \frame{\includegraphics[scale=.35]{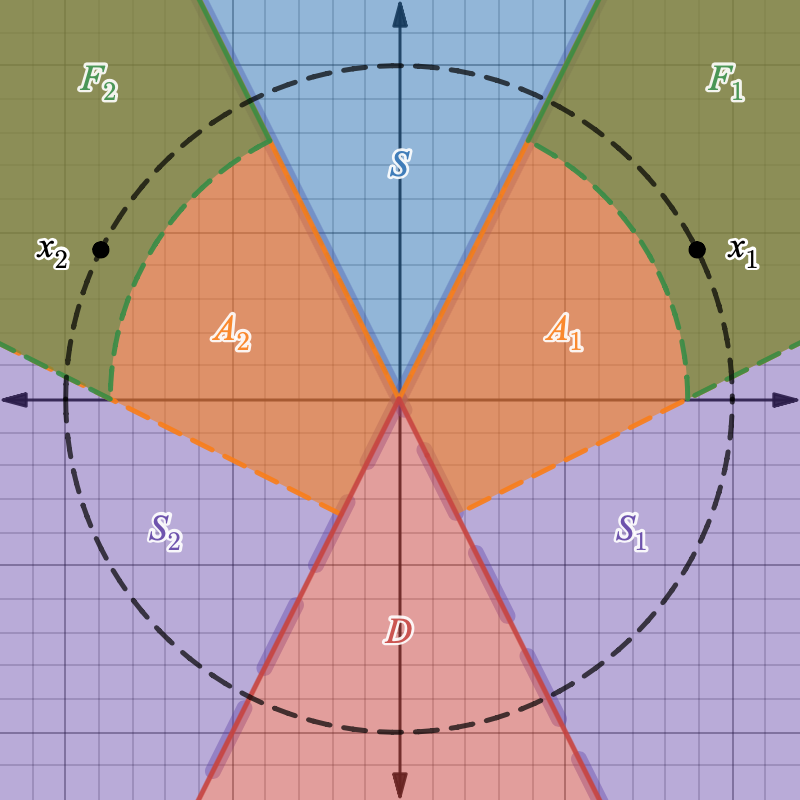}}
    \caption{Regions~$\cf_i$~(in~green) and~$\ca_i$ (the union of the green and orange regions). A dashed line marks an open~boundary.}
    \label{fig:F_and_A_regions}
\end{figure}

Furthermore, we show that~$\norm{\bw_i(\infty)}$ cannot be too~small, namely, obtaining a lower~bound on~$\norm{\bw_i(\infty)}$. First, a theorem from \cite{du2018algorithmic} implies that~$\norm{\bw_i(t)}^2 - \norm{\bv_i(t)}^2$ remains constant throughout the training. Since at the initialization both~$\norm{\bw_i(0)}$ and~$\norm{\bv_i(0)}$ are small, the consequence is that~$\norm{\bv_i(\infty)}$ is small if~$\norm{\bw_i(\infty)}$ is small.
Also, since~$N_{W(\infty),V(\infty)}$ attains zero~loss and~$\bw_i(\infty) \in \cs_i$ for all $i \in \{1,2\}$, then we have~$\by_i = \bv_i(\infty) (\bw_i(\infty)^\top \bx_i)$, namely, only the $i$-th~hidden~neuron contributes to the output of~$N_{W(\infty),V(\infty)}$ for the input~$\bx_i$. Since~$\norm{\by_i}=\norm{\bx_i}=1$, it is impossible that both~$\norm{\bw_i(\infty)}$ and~$\norm{\bv_i(\infty)}$ are small. Hence, we are able to obtain a lower~bound on~$\norm{\bw_i(\infty)}$, which implies that~$\bw_i(\infty)$ is in a region~$\cf_i$ which is illustrated in \figref{fig:F_and_A_regions}.

Finally, we show that since $\bw_1(\infty) \in \cf_1$ and $\bw_2(\infty) \in \cf_2$ then the angle between $\bw_1(\infty)$ and $\bw_2(\infty)$ is smaller than $\pi$, in~contradiction to $\bw_2(\infty) = \alpha \bw_1(\infty)$.

\subsection{Theorem~\ref{T3}}

We show that if the initialization is such that~$\bw_1(0) \in \cs_1 \setminus \partial \cs_1$ and~$\bw_2(0) \in \cs_2 \setminus \partial \cs_2$ (or, equivalently, that~$\bw_1(0) \in \cs_2 \setminus \partial \cs_2$ and~$\bw_2(0) \in \cs_1 \setminus \partial \cs_1$), then GF converges to a zero-loss~network, and~$\norm{\bw_1(\infty)}, \norm{\bw_2(\infty)}$, $\measuredangle(\bw_1(\infty),\bw_2(\infty))$ are in the required intervals. Since by simple geometric arguments we can show that the initialization satisfies this requirement with~probability at~least~$2 \cdot \left(\frac{\measuredangle(\bx_1,\bx_2)}{2\pi} \right)^2$, the theorem follows.

Indeed, suppose that~$\bw_1(0) \in \cs_1 \setminus \partial \cs_1$ and~$\bw_2(0) \in \cs_2 \setminus \partial \cs_2$. We argue that GF converges to a zero-loss~network and~$\norm{\bw_1(\infty)}, \norm{\bw_2(\infty)}, \measuredangle(\bw_1(\infty),\bw_2(\infty))$ are in the required intervals, as follows. 
By analyzing the dynamics of GF for such an initialization, we show that \hbox{for all $t$ and $i$} we have~$\frac{d}{dt} \bw_i(t) = C_i(t) \bx_i$ for some~$C_i(t) \geq 0$. Thus,~$\bw_i(t)$ moves only in the direction of~$\bx_i$, and~$\bw_i(t) \in \cs_i \setminus \partial \cs_i$ for~all~$t$.
Moreover, we are able to prove that these properties of the trajectories~$\bw_1(t)$ and~$\bw_2(t)$ imply that GF converges to a zero-loss~network~$N_{W(\infty),V(\infty)}$.
Then, by similar arguments to the proof of \thmref{T2} we have~$\bw_i(\infty) \in \cf_i$ for~all~$i \in \{1,2\}$, where~$\cf_i$ are the regions from \figref{fig:F_and_A_regions}, and it allows us to obtain the required bounds on~$\norm{\bw_1(\infty)}, \norm{\bw_2(\infty)}$, and~$\measuredangle(\bw_1(\infty),\bw_2(\infty))$.

\subsection{Theorems~\ref{thm:positive-square-loss} and~\ref{thm:positive-exp}}

The intuition for the proofs of both theorems can be roughly described as follows. If the dataset is realizable by a shallow~network where the Frobenius~norm of each layer is~$B$, then it is also realizable by a deep~network where the Frobenius~norm of each layer is~$B^*$, where~$B^*$ is much~smaller than~$B$. Moreover, if the network is sufficiently~deep then~$B^*$ is~not much~larger than~$1$. On the other hand, since for the input~$\bx_i$ with~$\norm{\bx_i} \leq 1$ the output of the network is of size at~least~$1$, then the average~spectral~norm of the layers is at~least~$1$. Hence, the average~ratio between the spectral and the Frobenius~norms cannot be too small.

We now describe the proof ideas in a bit more detail, starting with \thmref{thm:positive-square-loss}. 
We use the network~$N$ of width~$m$ and depth~$k$ to construct a network~$N'$ of width~$m' \geq m$ and depth~$k' > k$ as follows. The first~$k$~layers of~$N'$ are obtained by scaling the layers of~$N$ by a factor~$\alpha := \left( \frac{1}{B} \right)^{\frac{k'-k}{k'}}$. Since the output~dimension of~$N$ is~$1$, then the $k$-th hidden~layer of~$N'$ has width~$1$. Then, the network~$N'$ has~$k'-k$ additional layers of width~$1$, such that the weight in each of these layers is~$\beta := \left( \frac{1}{B} \right)^{-\frac{k}{k'}}$. Overall, given input~$\bx_i$, we have 
\[
    N'(\bx_i) = N(\bx_i) \cdot \alpha^k \cdot \beta^{k'-k} = N(\bx_i)~.
\]
We denote by $\btheta'$ the parameters of the network~$N'$.

Let~$\btheta^* = \left[W^*_1,\ldots,W^*_{k'} \right]$ be a global~optimum of Problem~\ref{eq:optimization problem square loss}. From the optimality of~$\btheta^*$ it is possible to show that the layers in~$\btheta^*$ must be balanced, namely,~$\norm{W^*_i}_F = \norm{W^*_j}_F$ for~all~$i,j \in [k']$. We denote by~$B^*$ the Frobenius~norm of the layers. From the global~optimality of~$\btheta^*$ we also have~$\norm{\btheta^*} \leq \norm{\btheta'}$. Hence, by a calculation we can obtain
\begin{equation*} \label{eq:bound B star}
	B^* \leq B^{\frac{k}{k'}}~.
\end{equation*}
Moreover, we show that since there~is~$i \in [n]$ with~$\norm{\bx_i} \leq 1$ and~$y_i \geq 1$, then
\[
	\frac{1}{k'} \sum_{i \in [k']} \norm{W^*_i}_\sigma \geq 1~.
\]
Combining the last two displayed equations we get
\[
	 \frac{1}{k'} \sum_{i \in [k']} \frac{\norm{W^*_i}_\sigma}{\norm{W^*_i}_F}
    	= \frac{1}{B^*} \cdot \frac{1}{k'} \sum_{i \in [k']} \norm{W^*_i}_\sigma
    	\geq \left( \frac{1}{B} \right)^{\frac{k}{k'}}~,
\]
as required.

Note that the arguments above do~not depend on the ranks of the layers in~$N$. Thus, even if the weight~matrices in~$N$ have high~ranks, once we consider deep~networks which are optimal~solutions to Problem~\ref{eq:optimization problem square loss}, the ratios between the spectral and the Frobenius~norms are close~to~$1$.

We now turn to \thmref{thm:positive-exp}. The proof follows a similar approach to the proof of \thmref{thm:positive-square-loss}. However, here the outputs of the network~$N$ can be either positive or negative. Hence, when constructing the network~$N'$ as above, we cannot have width~$1$ in layers~$k+1, \ldots, k'$, since the ReLU~activation will~not allow us to pass both positive and negative values. Still, we show that we can define a network~$N'$ such that the width in layers~$k+1, \ldots, k'$ is $2$ and we have~$N'(\bx_i)=N(\bx_i)$ for all $i \in [n]$. Then, the theorem follows by arguments similar to the proof of \thmref{thm:positive-square-loss}, with the required modifications.

\subsection*{Funding Acknowledgements}
This research is supported in part by European~Research~Council~(ERC) grant~754705.

\bibliography{bib}
\bibliographystyle{abbrvnat}

%%%%%%%%%%%%%%%%%%%%%%%%%%%%%%%%%%%%%%%%%%%%%%%%%%%%%%%%%%%%%%%%%%%%%%%%%%%%%%%
%%%%%%%%%%%%%%%%%%%%%%%%%%%%%%%%%%%%%%%%%%%%%%%%%%%%%%%%%%%%%%%%%%%%%%%%%%%%%%%
% APPENDIX
%%%%%%%%%%%%%%%%%%%%%%%%%%%%%%%%%%%%%%%%%%%%%%%%%%%%%%%%%%%%%%%%%%%%%%%%%%%%%%%
%%%%%%%%%%%%%%%%%%%%%%%%%%%%%%%%%%%%%%%%%%%%%%%%%%%%%%%%%%%%%%%%%%%%%%%%%%%%%%%
\newpage
\appendix
\onecolumn

\section{Proof of \thmref{T1}} \label{app:T1}

Consider a matrix $W \in \reals^{2 \times 2}$ whose rows $\bw_1^\top,\bw_2^\top$ satisfy
\begin{align*}
    \bw_1 &= \frac {\bx_1} {\| \bx_1 \|} - \frac {\bx_2} {\| \bx_2 \|}, \\
	\bw_2 &= - \bw_1.
\end{align*}
    
The matrix $W$ has rank $1$. To complete the proof, we need to show that we can choose a matrix $V \in \reals^{2 \times 2}$ such that $N_{W,V}$ attains zero~loss. According to \lemref{lemma:separation} below, it is enough to show that $\bw_1^\top \bx_1 >0$ and $\bw_1^\top \bx_2 < 0$. Since the angle between the inputs is strictly positive, namely, $\measuredangle(\bx_1, \bx_2) > 0$, it holds that $\frac {\bx_1^\top \bx_2} {\| \bx_1 \| \cdot \| \bx_2 \|} < 1$. Thus,
\[
       \| \bx_1 \| \cdot \| \bx_2 \| - {\bx_1^\top \bx_2} > 0~.
\]
Then,
\begin{align*}
    \bw_1^\top \bx_1 &= \frac {\bx_1^\top \bx_1} {\| \bx_1 \|} - \frac {\bx_1^\top \bx_2} {\| \bx_2 \|} 
    = \| \bx_1 \| - \frac {\bx_1^\top \bx_2} {\| \bx_2 \|} 
    = \frac {\| \bx_1 \| \cdot \| \bx_2 \| - \bx_1^\top \bx_2} {\| \bx_2 \|}
    > 0,
\end{align*}
while
\begin{align*}
	\bw_1^\top \bx_2 &= \frac {\bx_1^\top \bx_2} {\| \bx_1 \|} - \frac {\bx_2^\top \bx_2} {\| \bx_2 \|} 
    = \frac {\bx_1^\top \bx_2} {\| \bx_1 \|} - \| \bx_2 \|  
    = \frac {\bx_1^\top \bx_2 - \| \bx_1 \| \cdot \| \bx_2 \|} {\| \bx_1 \|} 
    < 0.
\end{align*}

\qed

\begin{lemma}\label{lemma:separation}
    Let $(X,Y) \in \mathbb{R}^{\din \times n} \times \mathbb{R}^{\dout \times n}$ be a labeled dataset. 
    Let $W \in \reals^{\dhid \times \din}$.
    Suppose that for every data point $\bx_j$ there is at least one row $\bw_i^\top$ in $W$ such that $\bw_i^\top \bx_j > 0$, and $\bw_i^\top \bx_\ell \leq 0$ for all $\ell \neq j$. Then, there exists $V$ such that $N_{W,V}(X) = Y$.
\end{lemma}

\begin{proof}
    Consider the matrix $\sigma (WX)$ of size $\dhid \times n$, where $\sigma$ acts entrywise. Note that our assumption on $W$ implies that $\rank \left( \sigma (WX) \right) = n$.
    Thus, the $\dhid \times \dhid$ matrix $Z := \begin{bmatrix} \sigma (WX)^{\dagger} \\ 0 \end{bmatrix}$ satisfies $Z \sigma (WX) = \begin{bmatrix} I_n \\ 0 \end{bmatrix}$, where $A^\dagger$ denotes the Moore-Penrose inverse of a matrix $A$, and $I_n$ is the $n \times n$ identity matrix. Hence, the matrix $M := \begin{bmatrix} Y & 0 \end{bmatrix}$ of dimensions $\dout \times \dhid$ yields $M Z \sigma (WX) = Y$. By setting $V := M Z$, the network~$N_{W,V}$ achieves zero~loss. Namely, $N_{W,V}(X) = Y$.
\end{proof}

\section{Proof of \thmref{T2}}

\begin{definition}\label{def:regions}
    We define the following regions of interest:
    \begin{align*}
        \mathcal{D} &:= \{ \bw \in \mathbb{R}^2 \mid \forall i \in \{1,2\}, \sigma(\bw^\top \bx_i) \le 0\}, \\
        \mathcal{S} &:= \{ \bw \in \mathbb{R}^2 \mid \forall i \in \{1,2\}, \sigma(\bw^\top \bx_i) > 0 \}. \\
    \end{align*}
    Also, for $j \in \{1,2\}$ we define
    \[
        \mathcal{S}_j := \{ \bw \in \mathbb{R}^2  \mid \sigma(\bw^\top \bx_j) > 0, \sigma(\bw^\top \bx_{3-j}) \leq 0\}~. 
    \]
\end{definition}

The regions in the above definition appear in \figref{fig:D_and_Ss_regions}.
Note that each of the regions of Definition \ref{def:regions}, denoted as $\mathcal{P}$, is disjoint from the others and satisfies $c \cdot p \in \mathcal{P}$ for all $p \in \mathcal{P}$ and $c \in \mathbb{R}$ where $c > 0$.
Assumption \ref{assumption:xs_bounded_angle} induces the following geometry: 
Each of the four regions $\mathcal{D}$, $\mathcal{S}_1$, $\mathcal{S}_2$ and $\mathcal{S}$ is nonempty, and 
any straight line on the plane that goes through the origin intersects exactly two regions: Either (i) the $\mathcal{S}$ and $\mathcal{D}$ regions, or (ii) one of the $\mathcal{S}_i$ regions and the $\mathcal{D}$ region, or (iii) the $\mathcal{S}_1 \setminus \partial \mathcal{S}_1$ and $\mathcal{S}_2 \setminus \partial \mathcal{S}_2$ regions. 

Assume, for the sake of contradiction, that GF converges to some zero-loss~network~$N_{W(\infty),V(\infty)}$ with~$\rank(W(\infty))<2$. 
On the one hand, in \lemref{lemma:rank-deficient-net} we show that the weight~vectors~$\bw_1(\infty)$ and $\bw_2(\infty)$ are non-zero, and satisfy $\bw_2(\infty) = \alpha \bw_1(\infty)$ with $\alpha < 0$. It implies that the straight line that connects $\bw_1(\infty)$ and $\bw_2(\infty)$, denoted as $\bw_1\bw_2$, goes through the origin. On the other hand, in \lemref{lemma:no-dead-neuron at infty} we show that $\bw_i(\infty) \not \in \mathcal{D}$ for every $i \in \{1,2\}$. In other words, $\bw_1\bw_2$ cannot intersect the $\mathcal{D} \setminus \set{\zero}$ region. Thus, one neuron must lie in $\mathcal{S}_1 \setminus \partial \mathcal{S}_1$ and the other neuron in $\mathcal{S}_2 \setminus \partial \mathcal{S}_2$. W.l.o.g., let $\bw_i(\infty) \in \mathcal{S}_i \setminus \partial \mathcal{S}_i$ for all $i \in \{1,2\}$. Therefore, by \lemref{lemma:ws_bounded_norm_and_angle}, it holds that $\measuredangle \big( \bw_1(\infty), \bw_2(\infty) \big) \in \Big[ \pi - \measuredangle(\bx_1, \bx_2), \measuredangle(\bx_1, \bx_2) + 2 \arcsin{\frac{2 \max_{i \in [2]} \norm{\bw_i(0)}}{\sqrt{3}}} \Big)$. To complete the proof by contradiction, it remains to show that 
$\measuredangle \big( \bw_1(\infty), \bw_2(\infty) \big) < \pi$
so that $\bw_2(\infty) \ne \alpha \bw_1(\infty)$.    
Recall that we initialize the network such that 
$\norm{\bw_i(0)} < \frac{\sqrt{3}}{2} \cos\big({\frac{\measuredangle(\bx_1, \bx_2)}{2}}\big) = \frac{\sqrt{3}}{2} \sin\big( \frac{\pi}{2} - \frac{\measuredangle(\bx_1, \bx_2)}{2} \big) $. 
Hence, $\measuredangle \big( \bw_1(\infty), \bw_2(\infty) \big) < \pi$, as required.

\qed

\begin{lemma}\label{lemma:rank-deficient-net}
    Let $(X,Y) \in \reals^{2 \times 2} \times \reals^{2 \times 2}$ be a labeled dataset that satisfies Assumption \ref{assumption:ys_linearly_independent}. Consider a zero-loss~ReLU~network~$N_{W,V}$ where $W,V \in \reals^{2 \times 2}$ and $\rank(W)<2$. Then, the weight~vectors~$\bw_1$ and $\bw_2$ are non-zero, and satisfy $\bw_2 = \alpha \bw_1$ with $\alpha < 0$.
\end{lemma}

\begin{proof}
    First, by \lemref{lemma:no-dead-neuron at infty} we have $\bw_1 \neq \zero$ and $\bw_2 \neq \zero$. Thus, $\rank(W) > 0$. Since by our assumption we also have $\rank(W) < 2$ then we must have $\rank(W) = 1$. Hence, we can denote $\bw_2 = \alpha \bw_1$ for some $\alpha \in \reals$ with $\alpha \neq 0$.
    
    Now, we prove that $\alpha < 0$. Assume for the sake of contradiction that $\alpha > 0$. 
    Then,
    we have $\sigma(\bw_2^\top \bx_j) = \alpha \sigma (\bw_1^\top \bx_j)$ for all $j \in [2]$. Thus, $\rank \left( \sigma \left( W X \right) \right) \leq 1$. Therefore, $\rank \left( V \sigma \left( W X \right) \right) \leq \min \{  \rank(V), \rank \left( \sigma \left( W X \right) \right) \} \leq 1$.
    Since by Assumption \ref{assumption:ys_linearly_independent} we have $\rank \left( Y \right) = 2$, then we conclude that $Y \ne V \sigma (WX)$, in~contradiction to the zero-loss assumption. Therefore, $\alpha < 0$, as required.
\end{proof}

\begin{lemma} \label{lemma:no-dead-neuron at infty}
    Let $(X,Y) \in \reals^{2 \times 2} \times \reals^{2 \times 2}$ be a labeled dataset that satisfies Assumption \ref{assumption:ys_linearly_independent}. Consider a zero-loss~ReLU~network~$N_{W,V}$ where $W,V \in \reals^{2 \times 2}$.
    Then, we have $\bw_i \not \in \mathcal{D}$ for all $i \in \{1,2\}$.
\end{lemma}
\begin{proof}
    Assume that there is $i \in [2]$ such that $\bw_i \in \mathcal{D}$. Hence, $\sigma(\bw_i^\top \bx_j)=0$ for all $j \in [2]$. 
    Thus, $\rank \left( \sigma \left( W X \right) \right) \leq 1$. Therefore, $\rank \left( V \sigma \left( W X \right) \right) \leq \min \{  \rank(V), \rank \left( \sigma \left( W X \right) \right) \} \leq 1$.
    Since by Assumption \ref{assumption:ys_linearly_independent} we have $\rank \left( Y \right) = 2$, then we conclude that $Y \ne V \sigma (WX)$, in~contradiction to the zero-loss assumption. 
\end{proof}

\begin{lemma}\label{lemma:no-dead-neuron}
    Let \((X, Y) \in \mathbb{R}^{2 \times 2} \times \mathbb{R}^{2 \times 2}\) be a labeled dataset that satisfies Assumption \ref{assumption:ys_linearly_independent}. 
    Consider GF w.r.t. the loss~function~$L_{X,Y}(W,V)$ for $W,V \in \reals^{2 \times 2}$, and assume that it converges to a network~$N_{W(\infty),V(\infty)}$.
    Suppose that there exist $i \in [2]$ and time $t \geq 0$ such that $\bw_i(t) \in \mathcal{D}$. Then, we have $N_{W(\infty),V(\infty)}(X) \ne Y$.
\end{lemma}
\begin{proof}
    Note that if $\bw_i(t) \in \mathcal{D}$ then the gradient of $L_{X,Y}$ w.r.t. $\bw_i$ is zero. Hence $\bw_i$ remains constant for all $t' \geq t$. Therefore, $\bw_i(\infty) \in \mathcal{D}$.
    The claim now follows from \lemref{lemma:no-dead-neuron at infty}.
\end{proof}

\begin{lemma}\label{lemma:ws_bounded_norm_and_angle}
    Let $(X,Y)$ be a labeled dataset that satisfies Assumptions \ref{assumption:ys_linearly_independent} and \ref{assumption:xs_bounded_angle}. 
    Consider GF w.r.t. the loss~function~$L_{X,Y}(W,V)$. 
    Suppose that $W,V \in \reals^{2 \times 2}$ are initialized such that for all $i \in [2]$ we have $\|\bw_i(0)\|, \|\bv_i(0)\| < \frac{1}{2}$.
    If GF converges to a zero-loss~network~$N_{W(\infty),V(\infty)}$ such that
    $\bw_i(\infty) \in \mathcal{S}_i \setminus \partial \mathcal{S}_i$ for all $i \in [2]$, then 
    \[
        \norm{\bw_i(\infty)} \in \left( \frac{\sqrt{3}}{2}, \sqrt{ \frac{1}{4} + \frac{4}{3 \left(\cos  \max{ \set{ \arcsin{\frac{2 \norm{\bw_i(0)}}{\sqrt{3}}}, \measuredangle\big(\bx_1, \bx_2\big) - \frac{\pi}{2} } } \right)^2 } } \right)
    \]
    and 
    \[
        \measuredangle (\bw_1(\infty), \bw_2(\infty)) \in \left[ \pi - \measuredangle(\bx_1, \bx_2), \measuredangle(\bx_1, \bx_2) + 2 \arcsin{\frac{2 \max_{i \in [2]} \norm{\bw_i(0)}}{\sqrt{3}}} \right)~.
    \]
\end{lemma}

\begin{proof}
    First, in \lemref{lemma:S_i-dynamics} we show that GF induces a dynamic on each neuron $\bw_i$ that lies in a $\mathcal{S}_i$ region, such that the neuron can only move in the direction of $\bx_i$. Formally, for every $\bw_i \in \mathcal{S}_i$ we have $\frac {d} {d t} \bw_i(t) = c_t^{(i)} \bx_i$, where $c_t^{(i)} \in \mathbb{R}$. We denote by $t_0^{(i)}$ the last time that $\bw_i$ enters $\mathcal{S}_i$. That is,
    
    \[
        t_0^{(i)} := \inf \{t \mid \bw_i(t') \in \mathcal{S}_i \text{ for all } t' \geq t\}~.
    \]
    Thus,
    \begin{equation} \label{eq:trajectory in Si}
        \bw_i(\infty) = \bw_i(t_0^{(i)}) + C^{(i)} \bx_i,
    \end{equation}
    for some constant $C^{(i)} \in \mathbb{R}$. 
    Note that since $\bw_i(\infty) \in \mathcal{S}_i \setminus \partial \mathcal{S}_i$, then there exists some $t' \geq 0$ with $\bw_i(t') \in \mathcal{S}_i$.
    We will further delimit the location of $\bw_i(t_0^{(i)})$. There are only two cases for $t_0^{(i)}$: 
    
    {\bf Case \(t_0^{(i)} = 0\):} If the last time that $\bw_i$ enters $\mathcal{S}_i$ is at initialization, then we have $t_0^{(i)} = 0$. Our assumptions on the initialization imply that: 
    \[
         \text{For } t_0^{(i)} =0, \quad \bw_i(t_0^{(i)}) \in \mathcal{T}_i := \mathcal{S}_i \cap \left( \cl B_2\left(\norm{\bw_i(0)}\right)\right)~.
    \]
    Note that by \lemref{lemma:no-dead-neuron} it is not~possible that $\bw_i(0) \in \mathcal{D}$, and hence we cannot have $\bw_i(0) \in \partial \mathcal{S}_i \cap \mathcal{D}$.
    
    {\bf Otherwise (i.e., $t_0^{(i)} >0$):} In that case, $t_0^{(i)}$ is when the neuron moves from some other region to $\mathcal{S}_i$. The other region can only be $\mathcal{S}$ or $\mathcal{D}$, due to the geometry that Assumption~\ref{assumption:xs_bounded_angle} imposes. Since \lemref{lemma:no-dead-neuron} implies that at any time no neuron is in $\mathcal{D}$, then the previous region is necessarily $\mathcal{S}$. Hence, we have:
    \[
        \text{For } t_0^{(i)} > 0, \quad \bw_i(t_0^{(i)}) \in \mathcal{U}_i := \partial(\mathcal{S}_i) \setminus \mathcal{D}~.
    \]
    
    In any case, we conclude that: 
    \[
        \bw_i(t_0^{(i)}) \in \mathcal{E}_i := \mathcal{T}_i \cup \mathcal{U}_i~.
    \]
    Therefore, the region of all neurons that are reachable under the aforementioned dynamics of GF is
    \begin{align*}
    	\bw_i(\infty) \in \mathcal{A}_i := \{ \bw + \lambda \bx_i \mid \bw \in \mathcal{E}_i, \lambda \ge 0 \}.
    \end{align*}
    
     We can assume that $\lambda \ge 0$ in the above definition, because every $\Bar{\ba} \in \{ \bw + \lambda \bx_i \mid \bw \in \mathcal{E}_i, \lambda < 0 \} \setminus \mathcal{A}_i$ satisfies $\Bar{\ba} \notin \mathcal{S}_i$. 
     
     We
     denote $\epsilon_{0}^{(i)} := \| \bw_i(0) \|^2 - \| \bv_i(0) \|^2$.
     By \lemref{lemma:du} we have $\epsilon_{0}^{(i)} = \| \bw_i(t) \|^2 - \| \bv_i(t) \|^2$ for any time $t \ge 0$, and hence $\epsilon_{0}^{(i)} = \| \bw_i(\infty) \|^2 - \| \bv_i(\infty) \|^2$.
     By \lemref{lemma:norm-at-convergance} we obtain $\| \bw_i(\infty) \| \ge \sqrt{1 - |\epsilon_0^{(i)}|}$ for every $i \in [2]$. 
     We define a new region of interest: The set of all feasible neurons at the convergence of GF, i.e., neurons that are reachable and satisfy the minimal norm requirement.
     Formally,
    \begin{align*}
    	\bw_i(\infty) \in \mathcal{F}_i := \Big\{ \bw  \in \mathcal{A}_i \Bigm| \| \bw \| \ge \sqrt{1 - |\epsilon_0^{(i)}|} \Big\} 
    	= \mathcal{A}_i \setminus B_2\left(\sqrt{1 - |\epsilon_0^{(i)}|}\right)~.
    \end{align*}
    
    The regions \(\mathcal{A}_i\) and \(\mathcal{F}_i\) are illustrated in Figure~\ref{fig:F_and_A_regions}.
    Recall that all neurons are initialized such that $\| \bw_i(0) \| , \| \bv_i(0) \| < \frac{1}{2}$ for all $i \in [2]$.
    Thus, we have
    $\left| \epsilon_0^{(i)} \right| < (\frac{1}{2})^2 = \frac{1}{4}$ for all $i \in [2]$. Hence, 

    \begin{equation} \label{eq:wi lower bound} 
        \| \bw_i(\infty) \| > \frac{\sqrt{3}}{2}~, 
    \end{equation}
    as required. 
    
    We now consider the angle between $\bw_1(\infty)$ and $\bw_2(\infty)$.
    On the one hand, the minimal angle between the neurons is achieved when $\bw_1(\infty)$ and $\bw_2(\infty)$ lie on the ``non-dead boundaries'' of $\mathcal{S}_1, \mathcal{S}_2$. That is,
    \begin{equation} \label{eq:angle lower bound}
        \measuredangle \big( \bw_1(\infty), \bw_2(\infty) \big) \ge \measuredangle(\bb_1, \bb_2) = \pi - \measuredangle(\bx_1, \bx_2)~,
    \end{equation} 
    where $\bb_i \in \partial(\mathcal{S}_i) \setminus \mathcal{D}$. 
    On the other hand, the angle between the neurons is maximized when 
    \[
        \measuredangle\big(\bw_1(\infty), \bw_2(\infty)\big) = \measuredangle(\bx_1, \bx_2) + \sum_{i=1}^{2} \measuredangle\big(\bw_i(\infty), \bx_i\big)~. 
    \] 
    Note that in the above expression the angle $\measuredangle\big(\bw_i(\infty), \bx_i\big)$ corresponds to the case where $\bw_i(\infty)$ is in the direction w.r.t. $\bx_i$ which is closer to $\mathcal{D}$ and farther from $\mathcal{S}$. 
    Due to \eqref{eq:trajectory in Si} and the definition of $\mathcal{F}_i$, the appropriate angle $\measuredangle\big(\bw_i(\infty), \bx_i\big)$ in the above expression can be upper~bounded by $\arcsin\frac{\norm{\bw_i(0)}}{\|  \bw_i(\infty) \|}$.
    It corresponds to the case where $\bw_i$ is initialized in $\mathcal{S}_i$ such that $\measuredangle(\bw_i(0),\bx_i)$ is close~to~$\pi/2$, and $\bw_i$ follows the trajectory from \eqref{eq:trajectory in Si}. Using \eqref{eq:wi lower bound} we have $\arcsin\frac{\norm{\bw_i(0)}}{\|  \bw_i(\infty) \|} < \arcsin\frac{2 \norm{\bw_i(0)}}{\sqrt{3}}$.
    Hence, we get 
    
    \[
        \measuredangle\big(\bw_1(\infty), \bw_2(\infty)\big)
        < \measuredangle(\bx_1, \bx_2) + 2 \arcsin\frac{2 \max_{i \in [2]} \norm{\bw_i(0)}}{\sqrt{3}}~.
    \]
    
    Combining the above with \eqref{eq:angle lower bound} we obtain 
    \begin{align*}
        \measuredangle\big(\bw_1(\infty), \bw_2(\infty)\big) \in \left[ \pi - \measuredangle(\bx_1, \bx_2), ~ \measuredangle(\bx_1, \bx_2) + 2 \arcsin\frac{2 \max_{i \in [2]} \norm{\bw_i(0)}}{\sqrt{3}} \right)~.
    \end{align*}
    
    Finally, we obtain an upper~bound for $\| \bw_i(\infty) \|$. We have
    $\bw_i(\infty)^\top \bx_i = \| \bw_i(\infty) \| \cdot \| \bx_i \| \cos{\measuredangle\big(\bw_i(\infty), \bx_i\big)} > \frac{\sqrt{3}}{2} \cos{\measuredangle\big(\bw_i(\infty), \bx_i\big)}$
    for all $i \in [2]$. Note that $\measuredangle\big(\bw_i(\infty), \bx_i\big)$ corresponds either to the case where $\bw_i(\infty)$ is in the direction w.r.t. $\bx_i$ which is closer to $\mathcal{D}$ and farther from $\mathcal{S}$, or closer to $\mathcal{S}$ and farther from $\mathcal{D}$. For the former case, we saw that $\measuredangle\big(\bw_i(\infty), \bx_i\big) < \arcsin{\frac{2 \norm{\bw_i(0)}}{\sqrt{3}}}$. In the latter case, $\measuredangle\big(\bw_i(\infty), \bx_i\big) = \measuredangle\big(\bx_1, \bx_2\big) - \measuredangle\big( \bw_i(\infty), \bx_{3-i} \big) \leq \measuredangle\big(\bx_1, \bx_2\big) - \frac{\pi}{2}$. Therefore, $\bw_i(\infty)^\top \bx_i > \frac{\sqrt{3}}{2} \cos \max{ \set{ \arcsin{\frac{2 \norm{\bw_i(0)}}{\sqrt{3}}}, \measuredangle\big(\bx_1, \bx_2\big) - \frac{\pi}{2} } }$. Since the network has zero-loss, i.e., it interpolates the entire dataset, then we have that \(\bv_i(\infty) = \frac{1}{\bw_i(\infty)^\top \bx_i} \by_i\). Hence, 
    \[
        \| \bv_i(\infty) \| 
        = \frac{1}{\bw_i(\infty)^\top \bx_i} \| \by_i \| 
        < \frac{2}{\sqrt{3} \cos \max{ \set{ \arcsin{\frac{2 \norm{\bw_i(0)}}{\sqrt{3}}}, \measuredangle\big(\bx_1, \bx_2\big) - \frac{\pi}{2} } } }~. 
    \]
    By Lemma \ref{lemma:du}, we have $\| \bw_i(\infty) \|^2 - \| \bv_i(\infty) \|^2 = \| \bw_i(0) \|^2 - \| \bv_i(0) \|^2 < \frac{1}{4}$. Therefore,
    \[
        \norm{\bw_i(\infty)}^2 < \frac{1}{4} + \frac{4}{3 \left(\cos  \max{ \set{ \arcsin{\frac{2 \norm{\bw_i(0)}}{\sqrt{3}}}, \measuredangle\big(\bx_1, \bx_2\big) - \frac{\pi}{2} } } \right)^2 },
    \]
    as required.
\end{proof}

\begin{lemma}\label{lemma:S_i-dynamics}
    Let $(X, Y) \in \mathbb{R}^{2 \times 2} \times \mathbb{R}^{2 \times 2}$ be a labeled dataset that satisfies Assumption~\ref{assumption:xs_bounded_angle}. Consider GF on a ReLU~network~$N_{W,V}$ with $W,V \in \reals^{2 \times 2}$, w.r.t. $L_{X,Y}(W,V)$. Assume that at time $t$ we have $\bw_i(t) \in \mathcal{S}_i$ for some $i \in [2]$. Then, there exists $c_t^{(i)} \in \mathbb{R}$ such that $\frac {d} {d t} \bw_i(t) = c_t^{(i)} \bx_i$.
\end{lemma}

\begin{proof}
    We have
    \begin{align*}
    	\frac {d} {d t} \bw_i(t) 
    	&= - \frac \partial {\partial \bw_i} L_{X,Y}(W(t),V(t))~. 
    \end{align*}
    The derivative of the $L_{X,Y}$ w.r.t. the matrix $W$ is
    \[
        \frac \partial {\partial W} L_{X,Y}\left(W,V\right) = \Bigg( \sigma'\left(WX\right) \odot \Big( V^\top \big( V \sigma (WX) - Y \big) \Big) \Bigg) X^\top~.
    \]
    Here, $\odot$ denotes the Hadamard product (i.e., the entrywise product). Note that $\frac {\partial L_{X,Y}\left(N_{W,V}\right)}{\partial W}$ is a matrix whose $(i,j)$-th entry is $\frac{\partial L_{X,Y}\left(W,V\right)}{\partial W_{i,j}}$. We denote the $i$-th row of $\sigma'\left(WX\right)$ by $\sigma'\left(WX\right)_i$. We have
    \begin{align*}
    	\sigma'\big( WX \big)_i 
    	= \sigma'\big( \begin{bmatrix} \bw_i^\top \bx_1 & \bw_i^\top \bx_2 \end{bmatrix} \big) 
        = \begin{bmatrix} \sigma'(\bw_i^\top \bx_1) & \sigma'(\bw_i^\top \bx_2) \end{bmatrix}.
    \end{align*}
    
    If
    $\bw_i \in \mathcal{S}_i$ then the $j$-th entry of the aforementioned row vector is
    \[ 
        \sigma'\left(WX\right)_{ij} = \onefunc \{ i = j \}~. 
    \]
    Thus, there exists a constant \(\alpha^{(i)} \in \mathbb{R}\) such that
    \[ 
        \Bigg( \sigma'\left(WX\right) \odot \Big( V^\top \big( V \sigma (WX) - Y \big) \Big) \Bigg)_{ij} = \onefunc \{ i = j \} \alpha^{(i)}~. 
    \]
    
    Since the derivative of the loss w.r.t. the \(i\)-th neuron $\bw_i$ is the \(i\)-th row of $\frac{\partial}{\partial W} L_{X,Y}\left(W,V\right)$, we conclude that
    \[
        \frac \partial {\partial \bw_i} L_{X,Y}(W,V) = \alpha^{(i)} \bx_i~.
    \]
    By setting $c_t^{(i)} = - \alpha^{(i)}$, the proof is done. 
\end{proof}

\begin{lemma}\label{lemma:norm-at-convergance}
    Let $(X,Y) \in \reals^{2 \times 2} \times \reals^{2 \times 2}$ be a labeled dataset that satisfies Assumptions~\ref{assumption:ys_linearly_independent} and~\ref{assumption:xs_bounded_angle}. 
    Let $N_{W,V}$ be a zero-loss~network with~$W,V \in \reals^{2 \times 2}$, such that~$\bw_i \in \mathcal{S}_i$ for~all~$i \in [2]$.
    Let $\epsilon^{(i)} := \| \bw_i \|^2 - \| \bv_i \|^2$.
    Then $\| \bw_i \|, \|\bv_i\| \ge \sqrt{1 - |\epsilon^{(i)}|}$ for all $i \in [2]$.
\end{lemma}

\begin{proof}
    Since the network has zero~loss, for all $i \in [2]$ we have
    \begin{align*}
    	\by_i 
    	= V \sigma(W \bx_i) 
    	= \begin{bmatrix} 
          	\sum_{k=1}^{2} V_{1,k}\sigma(\bw_k^\top \bx_i)  \\
          	\\
          	\sum_{k=1}^{2} V_{2,k}\sigma(\bw_k^\top \bx_i) \\
        	\end{bmatrix}~.
    \end{align*}
    
    Since $\bw_i \in \mathcal{S}_i$ for every $i \in [2]$, we have $\sigma(\bw_k^\top \bx_i) = \begin{cases} \bw_i^\top \bx_i & \text{if } k=i \\ 0 & \text{otherwise} \end{cases}$. Hence, the above expression is equal to
    \begin{align*}
        \begin{bmatrix}
            V_{1,i} \cdot \bw_i^\top \bx_i \\ 
            \\ 
            V_{2,i} \cdot \bw_i^\top \bx_i \\
        \end{bmatrix} = \bv_i ( \bw_i^\top \bx_i )~.
    \end{align*}
    
    Therefore,
    \begin{align*}
    	1 = \| \by_i \| 
    	=  \| \bv_i ( \bw_i^\top \bx_i ) \|
    	=  \| (\bv_i  \bw_i^\top) \bx_i \| 
    	\le \| \bv_i \bw_i^\top \|_F \cdot \| \bx_i \|
    	= \| \bv_i \bw_i^\top \|_F 
    	= \| \bv_i \| \cdot \| \bw_i \|~.
    \end{align*}
    
    Now, there are two cases:
    
    {\bf Case \(\|\bw_i\| \le \|\bv_i\|\):} We have that $\|\bv_i\|^2 \ge 1$.
    Then, 
    \begin{align*}
    	\| \bw_i \| 
    	= \sqrt{\|\bv_i\|^2 + \epsilon^{(i)}} 
        \ge \sqrt{1 + \epsilon^{(i)}} 
        = \sqrt{1 - |\epsilon^{(i)}|}~.
    \end{align*}
    
    {\bf Otherwise:} Similarly, we have $\|\bw_i\|^2 \ge 1$. Then,
    \begin{align*}
    	\| \bv_i \| 
    	= \sqrt{\|\bw_i\|^2 - \epsilon^{(i)}} 
        \ge \sqrt{1 - \epsilon^{(i)}} 
        = \sqrt{1 - |\epsilon^{(i)}|}~.
    \end{align*}
    
    In any case, we conclude that
    \[
        \| \bw_i \|, \|\bv_i\| \ge \sqrt{1 - |\epsilon^{(i)}|}~.
    \]
\end{proof}

\begin{lemma}[\cite{du2018algorithmic}]\label{lemma:du}
    Let $N_{\btheta}$ be a fully-connected~depth-$k$~ReLU~network, where $k > 1$. 
    Denote $\btheta = [W^{(1)},\ldots,W^{(k)}]$.
    Consider minimizing any differentiable~loss~function (e.g., the square~loss) over a dataset using GF.
    Then, for every $l \in [k-1]$ at all time $t$ we have 
    \[
        \frac{d}{d t} \left( \norm{W^{(l)}(t)}_F^2 - \norm{W^{(l+1)}(t)}_F^2 \right) = 0~.
    \]
    Moreover, for every $l \in [k-1]$ and $i \in [d_l]$ at all time $t$ we have 
        \[
            \frac{d}{d t} \left( \norm{W^{(l)}[i, :](t)}^2 - \norm{W^{(l+1)}[:, i](t)}^2 \right) = 0~,
        \]
    where $W^{(l)}[i,:]$ is the vector of incoming weights to the $i$-th neuron in the $l$-th hidden~layer (i.e., the $i$-th row of $W^{(l)}$), and $W^{(l+1)}[:,i]$ is the vector of outgoing weights from this neuron (i.e., the $i$-th column of $W^{(l+1)}$).
\end{lemma}

\section{Proof of \thmref{T3}}

Consider the partition of $\reals^2$ into regions as described in Definition~\ref{def:regions}. 
If $\bw_i \in \mathcal{S}_i \setminus \partial \mathcal{S}_i$ for all $i \in [2]$, then the gradient of $L_{X,Y}(W,V)$ is given by:
\begin{align} \label{eq:gradients in Si}
    \frac{\partial}{\partial \bv_i} L_{X,Y} &= \left( \bw_i^\top \bx_i \right) \phi_i ~, \nonumber \\
    \frac{\partial}{\partial \bw_i} L_{X,Y} &= \bv^\top_i \phi_i \bx_i ~,
\end{align}
for all $i \in [2]$, where $\phi_i := (\bw_i^\top \bx_i) \bv_i - \by_i = N_{W,V}(\bx_i) - \by_i$.
We denote the parameters of the network by $\btheta = [W,V]$.
Moreover, when $\bw_i \in \mathcal{S}_i \setminus \partial \mathcal{S}_i$ for all $i \in [2]$ we denote $L^i_{X,Y}(\btheta) = \frac{1}{2} \norm{\phi_i}^2$. Then, we have $L_{X,Y}(\btheta) = \sum_{i=1}^2 L^i_{X,Y}(\btheta)$.

\begin{lemma} \label{lem:decreasing loss}
    Let $t_1>0$ and suppose that for all $t \in [0,t_1]$ and $i \in [2]$ we have $\bw_i(t) \in \mathcal{S}_i \setminus \partial \mathcal{S}_i$, and that $\bv_i(0)=\zero$. Then, we have $L^i_{X,Y}(\btheta(t_1)) < L^i_{X,Y}(\btheta(0))$. Moreover, for every time $t$ where $\bw_i(t) \in \mathcal{S}_i \setminus \partial \mathcal{S}_i$ for all $i \in [2]$ we have $\frac{d}{dt} L^i_{X,Y}(\btheta(t)) \leq 0$.    
\end{lemma}
\begin{proof}
    For
    time $t$ such that $\bw_i(t) \in \mathcal{S}_i \setminus \partial \mathcal{S}_i$ for all $i \in [2]$
    we denote $F_i(t):= L^i_{X,Y}(\btheta(t)) = \frac{1}{2} \norm{\phi_i(t)}^2$. Let $\btheta_i := [\bw_i,\bv_i]$.
    We have
    \begin{align} \label{eq:Fi decreasing}
        \frac{d}{dt} F_i(t)
        &= \left(\nabla_\btheta L^i_{X,Y}(\btheta(t))\right)^\top \frac{d \btheta(t)}{dt} \nonumber
        \\
        &= \left(\nabla_{\btheta_i} L^i_{X,Y}(\btheta(t))\right)^\top \frac{d \btheta_i(t)}{dt} \nonumber
        \\
        &= - \left(\nabla_{\btheta_i} L^i_{X,Y}(\btheta(t))\right)^\top \nabla_{\btheta_i} L_{X,Y}(\btheta(t)) \nonumber
        \\
        &= - \left(\nabla_{\btheta_i} L^i_{X,Y}(\btheta(t))\right)^\top \nabla_{\btheta_i} L^i_{X,Y}(\btheta(t)) \nonumber
        \\
        &= - \norm{\nabla_{\btheta_i} L^i_{X,Y}(\btheta(t))}^2~,
    \end{align}
    where we used the fact that $L^i_{X,Y}(\btheta)$ depends only on $\btheta_i$.
    Therefore, $\frac{d}{dt} L^i_{X,Y}(\btheta(t)) \leq 0$.
    
    Note that $\norm{\nabla_{\btheta_i} L^i_{X,Y}(\btheta(t))}^2$ is continuous as a function of $t$, and at time $0$ we have
    \begin{align*}
        \norm{\nabla_{\btheta_i} L^i_{X,Y}(\btheta(0))}
        &= \norm{\nabla_{\btheta_i} L_{X,Y}(\btheta(0))}
        \geq \norm{ \frac{\partial}{\partial \bv_i} L_{X,Y}(\btheta(0))}
        = \norm{(\bw_i^\top(0) \bx_i) \phi_i(0)}
        \\
        &= \norm{(\bw_i^\top(0) \bx_i) \left((\bw_i^\top(0) \bx_i)\bv_i(0) - \by_i \right)}
        = \norm{(\bw_i^\top(0) \bx_i) \left( - \by_i \right)}
        > 0~,
    \end{align*}
    where the last inequality is since $\bw_i^\top(0) \bx_i > 0$ and $\by_i \neq \zero$. Combining the above with \eqref{eq:Fi decreasing}, we conclude that there is some small enough $t_0 \in (0, t_1)$ such that for all $t \in [0,t_0]$ we have $\frac{d}{dt} F_i(t) < 0$. Moreover, \eqref{eq:Fi decreasing} implies that for all $t \in [t_0,t_1]$ we have $\frac{d}{dt} F_i(t) \leq 0$. Hence, $F_i(t_1) \leq F_i(t_0) < F_i(0)$. 
\end{proof}

\begin{lemma} \label{lem:defining t'}
    Suppose that we initialize $\btheta(0)$ such that $\bw_i(0) \in \mathcal{S}_i \setminus \partial \mathcal{S}_i$ and $\bv_i(0)=\zero$ for all $i \in [2]$.
    For every sufficiently small $t'>0$ we have for every $t \in [0,t']$ and $i \in [2]$ that $\bw_i(t) \in \mathcal{S}_i \setminus \partial{S}_i$, and at time $t'$ we have
    $\bv_i^\top(t') \phi_i(t') < 0$ and $\bv_i(t') \in \spn\{\by_i\}$. Moreover, $L^i_{X,Y}(\btheta(t')) < L^i_{X,Y}(\btheta(0))$.
\end{lemma}
\begin{proof}
    Since $\bw_i(0)$ is in the open set $\mathcal{S}_i \setminus \partial \mathcal{S}_i$ for all $i \in [2]$, then for every small enough $t > 0$ we have $\bw_i(t) \in \mathcal{S}_i \setminus \partial \mathcal{S}_i$. Also, by \lemref{lem:decreasing loss}, for every small enough $t>0$ we have $L^i_{X,Y}(\btheta(t)) < L^i_{X,Y}(\btheta(0))$. Let $\tilde{t}$ be such that the two conditions above hold for all $t \in \left( 0, \tilde{t} ~ \right]$.
    
    Let $g_i(t) := \bv_i^\top(t) \phi_i(t)$. We show that for every small enough $0 < t' < \tilde{t}$ we have $g_i(t')<0$. First, note that $g_i(0) = 0$ since $\bv_i(0) = \zero$. Moreover, $g_i(t)$ is continuously differentiable, and satisfies 
    \[
        \frac{d}{dt} g_i(t)
        = \left( \frac{d}{d t} \bv^\top_i(t) \right) \phi_i(t) + \bv^\top_i(t) \left( \frac{d}{d t} \phi_i(t) \right)
        = - \left( \bw_i^\top(t) \bx_i \right)  \phi_i^\top(t) \phi_i(t) + \bv^\top_i(t) \left( \frac{d}{d t} \phi_i(t) \right)~.
    \]
    Therefore, 
    \[
        \frac{d}{dt} g_i(0) = 
        - \left( \bw_i^\top(0) \bx_i \right)  \norm{\phi_i(0)}^2 + \bv^\top_i(0) \left( \frac{d}{d t} \phi_i(0) \right)
        = - \left( \bw_i^\top(0) \bx_i \right)  \norm{\phi_i(0)}^2~.
    \]
    Since $\bw_i(0) \in \mathcal{S}_i \setminus \partial \mathcal{S}_i$ then $\bw_i^\top(0) \bx_i >0$ and hence we obtain $\frac{d}{dt} g_i(0) < 0$.
    
    Overall, the function $g_i$ is continuously differentiable with $g_i(0)=0$ and $\frac{d}{dt} g_i(0) < 0$ and therefore we have $g_i(t')<0$ for every small enough $t'>0$.
    
    It remains to show that $\bv_i(t') \in \spn\{\by_i\}$. Since for every $t \in [0,t']$ we have $\bw_i(t) \in \mathcal{S}_i \setminus \partial \mathcal{S}_i$, then for every $t \in [0,t']$ we have
    \[
        \frac{d}{dt}\bv_i(t)
        = - (\bw_i^\top(t) \bx_i) \phi_i(t)
        = - (\bw_i^\top(t) \bx_i) \left( (\bw_i^\top(t) \bx_i) \bv_i(t) - \by_i \right)
        \in \spn\{\bv_i(t),\by_i\}~.
    \]
    Since the above holds for all $t \in [0,t']$ and $\bv_i(0) = \zero$, then for all $t \in [0,t']$ we have $\bv_i(t) \in \spn\{\by_i\}$. Thus, $\bv_i$ remains on the line $\spn\{\by_i\}$.
\end{proof}

\begin{lemma} \label{lem:remain G}
    Suppose that we initialize $\btheta(0)$ such that $\bw_i(0) \in \mathcal{S}_i \setminus \partial \mathcal{S}_i$ and $\bv_i(0)=\zero$ for all $i \in [2]$.
    Let $t'>0$ as in \lemref{lem:defining t'}, and denote $\bw'_i := \bw_i(t')$ for $i \in [2]$.
    Let 
    \begin{align*}
        G
        := \big\{
            \btheta~:~
            &\text{for all }i \in [2] \text{ we have} \\
            &\bw_i = \bw'_i + c_i \bx_i \text{ for }c_i\geq 0~,~\\
            &\bv_i \in \spn\{\by_i\}~,~\\
            &L^i_{X,Y}(W,V) \leq L^i_{X,Y}(W(t'),V(t')) < L^i_{X,Y}(W(0),V(0))~,~\\
            &\bv_i^\top \phi_i \leq 0 \big\}~.
    \end{align*}
    Then, for all $t \geq t'$ we have $\btheta(t) \in G$.
    
    Moreover, for all $t_2 \geq t_1 \geq t'$ and all $i \in [2]$ we have 
    \[
        \bw_i^\top(t_2) \bx_i
        \geq \bw_i^\top(t_1) \bx_i > 0~.
    \]
\end{lemma}
\begin{proof}
    By \lemref{lem:defining t'} we have $\btheta(t') \in G$. 
    Let $t \geq t'$ and suppose that $\btheta(t) \in G$. Note that for all $i \in [2]$ we have $\bw_i(t) = \bw'_i + c_i(t) \bx_i$ for some $c_i(t) \geq 0$. Since $\bw'_i \in \mathcal{S}_i \setminus \partial \mathcal{S}_i$ then we also have $\bw_i(t) \in \mathcal{S}_i \setminus \partial \mathcal{S}_i$. Hence, 
    \begin{equation} \label{eq:wi derivative}
        \frac{d}{dt} \bw_i(t)
        = - \frac{\partial}{\partial \bw_i} L_{X,Y}(\btheta(t)) 
        = - \bv_i^\top(t) \phi_i(t) \bx_i~.
    \end{equation}
    Since by the definition of $G$ we have $\bv_i^\top(t) \phi_i(t) \leq 0$ then the above can be written as $c'_i(t) \bx_i$ for some $c'_i(t) \geq 0$. Moreover, 
    \[
        \frac{d}{dt} \bv_i(t)
        = - \frac{\partial}{\partial \bv_i} L_{X,Y}(\btheta(t))
        = - (\bw_i^\top(t) \bx_i) \phi_i(t)
        = - (\bw_i^\top(t) \bx_i) \left((\bw_i^\top(t) \bx_i)\bv_i(t) - \by_i \right)~.
    \]
    Since by the definition of $G$ we have $\bv_i(t) \in \spn\{\by_i\}$, then the above is also in $\spn\{\by_i\}$.
    
    Moreover, by \lemref{lem:decreasing loss} we have $\frac{d}{dt} L^i_{X,Y}(\btheta(t)) \leq 0$.

    The above observations imply that as long as $\bv_i^\top(t) \phi_i(t) \leq 0$ the parameters $\bw_i(t)$ and $\bv_i(t)$ satisfy the conditions in $G$. We now show that if $\bv_i^\top(t) \phi_i(t) = 0$ then $\frac{d}{dt}\bw_i(t) = \frac{d}{dt}\bv_i(t) = \zero$, and hence GF will get stuck at $\bw_i(t),\bv_i(t)$. Thus, GF cannot reach $\bw_i,\bv_i$ with $\bv_i^\top \phi_i > 0$. 
    
    Suppose that $\bv_i^\top(t) \phi_i(t) = 0$, $\bv_i(t) \in \spn\{\by_i\}$, and $L^i_{X,Y}(\btheta(t)) \leq L^i_{X,Y}(\btheta(t')) < L^i_{X,Y}(\btheta(0))$. 
    Note that $\bv_i(t) \neq \zero$, since otherwise we have 
    \begin{align*}
        L^i_{X,Y}(\btheta(t)) 
        &= \frac{1}{2} \norm{\phi_i(t)}^2
        = \frac{1}{2} \norm{(\bw_i^\top(t) \bx_i) \zero - \by_i}^2
        = \frac{1}{2} \norm{(\bw_i^\top(0) \bx_i) \zero - \by_i}^2
        \\
        &= \frac{1}{2} \norm{(\bw_i^\top(0) \bx_i) \bv_i(0) - \by_i}^2
        = L^i_{X,Y}(\btheta(0))~, 
    \end{align*}
    in~contradiction to our assumption. 
    Now, since $\bv_i(t) \in \spn\{\by_i\}$, then $\phi_i(t) = (\bw_i^\top(t) \bx_i) \bv_i(t) - \by_i \in \spn\{\by_i\}$. 
    Thus, both $\bv_i(t)$ and $\phi_i(t)$ are in $\spn\{\by_i\}$, and we have $\bv_i(t) \neq \zero$ and $\bv_i^\top(t) \phi_i(t) = 0$. Therefore, $\phi_i(t) = \zero$.
    By \eqref{eq:gradients in Si} it implies that $\frac{d}{dt} \bw_i(t) = \frac{d}{dt} \bv_i(t) = \zero$.
    
    Thus, $\btheta(t) \in G$ for all $t \geq t'$.
    It remains to show that for all $t_2 \geq t_1 \geq t'$ and all $i \in [2]$ we have $\bw_i^\top(t_2) \bx_i \geq \bw_i^\top(t_1) \bx_i$. By \eqref{eq:wi derivative} and since $\bv_i^\top(t) \phi_i(t) \leq 0$ for all $t \geq t'$, we can write $\bw_i(t_1) = \bw'_i + \gamma_1 \bx_i$ and $\bw_i(t_2) = \bw'_i + \gamma_2 \bx_i$ where $\gamma_2 \geq \gamma_1 \geq 0$. Therefore
    \[
        \bw_i^\top(t_2) \bx_i
        = \bw'^\top_i \bx_i + \gamma_2 \norm{\bx_i}^2
        \geq \bw'^\top_i \bx_i + \gamma_1 \norm{\bx_i}^2
        = \bw_i^\top(t_1) \bx_i
        >0~.
    \]
\end{proof}

\begin{lemma} \label{lem:GF converges}
    Suppose that we initialize $\btheta(0)$ such that $\bw_i(0) \in \mathcal{S}_i \setminus \partial \mathcal{S}_i$ and $\bv_i(0)=\zero$ for all $i \in [2]$.
    Then, GF converges (i.e., $W(\infty)$ and $V(\infty)$ exist) and $L_{X,Y}(W(\infty),V(\infty))=0$. Moreover $\bw_i(\infty) \in \mathcal{S}_i \setminus \partial \mathcal{S}_i$ for all $i \in [2]$
\end{lemma}
\begin{proof}
    By \lemref{lem:remain G}, there is $t'>0$ such that for all $i \in [2]$ and $t \geq t'$ we have $\bw_i(t) = \bw_i(t') + c_i(t) \bx_i$ for $c_i(t) \geq 0$. Hence, $\bw_i(t) \in \mathcal{S}_i \setminus \partial \mathcal{S}_i$ for all $t \geq t'$. We have $\frac{d}{dt} L_{X,Y}(\btheta(t)) = \left(\nabla L_{X,Y}(\btheta(t))\right)^\top \frac{d}{dt}\btheta(t) = - \norm{\nabla L_{X,Y}(\btheta(t))}^2$. Hence, for $T \geq t'$ we have
    \[
        L_{X,Y}(\btheta(T))
        = L_{X,Y}(\btheta(t')) + \int_{t=t'}^T \frac{d}{dt} L_{X,Y}(\btheta(t)) dt
        = L_{X,Y}(\btheta(t')) - \int_{t=t'}^T \norm{\nabla L_{X,Y}(\btheta(t))}^2 dt~.
    \]
    Therefore, 
    \[
        \int_{t=t'}^T \norm{\nabla L_{X,Y}(\btheta(t))}^2 dt
        = L_{X,Y}(\btheta(t')) - L_{X,Y}(\btheta(T))
        \leq L_{X,Y}(\btheta(t'))~.
    \]
    Since it holds for every $T \geq t'$, then we have 
    \begin{equation} \label{eq:bounded int} 
        \int_{t=t'}^\infty \norm{\nabla L_{X,Y}(\btheta(t))}^2 dt \leq L_{X,Y}(\btheta(t')) < \infty~.
    \end{equation}
    
    Moreover, since $\bw_i(t) \in \mathcal{S}_i \setminus \partial \mathcal{S}_i$ for all $i \in [2]$ and $t \geq t'$, then by \eqref{eq:gradients in Si} we have 
    \begin{align*}
        L_{X,Y}(\btheta(t)) 
        &= \frac{1}{2} \sum_{i=1}^2 \norm{\phi_i(t)}^2
        = \frac{1}{2} \sum_{i=1}^2 \frac{1}{\left(\bw_i^\top(t) \bx_i\right)^2} \norm{\frac{\partial}{\partial \bv_i} L_{X,Y}(\btheta(t))}^2
        \\
        &\leq \left( \frac{1}{2} \sum_{i=1}^2 \frac{1}{\left(\bw_i^\top(t) \bx_i\right)^2} \right) \norm{\nabla L_{X,Y}(\btheta(t))}^2~.
    \end{align*}
    By \lemref{lem:remain G} we have $\left(\bw_i^\top(t) \bx_i \right)^2 \geq \left(\bw_i^\top(t') \bx_i \right)^2$. Therefore
    \[
        L_{X,Y}(\btheta(t)) 
        \leq \left( \frac{1}{2} \sum_{i=1}^2 \frac{1}{\left(\bw_i^\top(t') \bx_i\right)^2} \right) \norm{\nabla L_{X,Y}(\btheta(t))}^2~.
    \]
    Letting $K := \frac{1}{2} \sum_{i=1}^2 \frac{1}{\left(\bw_i^\top(t') \bx_i\right)^2}$ and combining the above with \eqref{eq:bounded int}, we get
    \[
        \frac{1}{K} \int_{t=t'}^\infty L_{X,Y}(\btheta(t)) dt  
        < \infty~.
    \]
    Since $L_{X,Y}(\btheta(t))$ is non-negative, and since by \lemref{lem:decreasing loss} it is monotonically non-increasing as a function of $t$, then we conclude that $\lim_{t \to \infty} L_{X,Y}(\btheta(t)) = 0$.
    
    It remains to show that $\btheta(\infty)$ exists, namely, that GF converges. Since $\bw_i(t) \in \mathcal{S}_i \setminus \partial \mathcal{S}_i$ for all $t \geq t'$ and $\lim_{t \to \infty} L_{X,Y}(\btheta(t)) = 0$, then $\lim_{t \to \infty} L^i_{X,Y}(\btheta(t)) = 0$ for all $i \in [2]$. That is, $(\bw_i^\top(t) \bx_i) \bv_i(t) \to \by_i$ as $t \to \infty$.
    By \lemref{lem:remain G} we can write $\bw_i(t) = \bw'_i + a_i(t) \bx_i$ and $\bv_i(t) = b_i(t) \by_i$, for some $a_i(t),b_i(t)$ with $a_i(t) \geq 0$ for all $t$. Since $\bw_i^\top(t) \bx_i > 0$ and $(\bw_i^\top(t) \bx_i) \bv_i(t) \to \by_i$ then we also have $b_i(t) > 0$ for large enough $t$. 
    
    By \lemref{lemma:du}, $\norm{\bv_i(t)}^2 - \norm{\bw_i(t)}^2$ remains constant throughout the training. Hence, we can write
    \[
        b_i(t)
        = \norm{b_i(t) \by_i}
        = \norm{\bv_i(t)}
        = \sqrt{C + \norm{\bw_i(t)}^2}
        = \sqrt{C + \norm{\bw'_i + a_i(t) \bx_i}^2}~,
    \]
    for some constant $C$.
    Therefore,
    \[
        (\bw_i^\top(t) \bx_i) \bv_i(t)
        = \left(\bw'^\top_i \bx_i + a_i(t) \norm{\bx_i}^2\right) \sqrt{C + \norm{\bw'_i + a_i(t) \bx_i}^2} \cdot \by_i~.
    \]
    Since $(\bw_i^\top(t) \bx_i) \bv_i(t) \to \by_i$, then we conclude that for 
    \[
        g_i(a) := \left(\bw'^\top_i \bx_i + a \norm{\bx_i}^2\right) \sqrt{C + \norm{\bw'_i + a \bx_i}^2} 
    \]
    we have $\lim_{t \to \infty} g_i(a_i(t)) = 1$. The function $g_i(a)$ on $[0,\infty)$ is continuous and strictly increasing, and $\lim_{a \to \infty}g(a) = \infty$. 
    Also, $g(0) \leq 1$ since otherwise we cannot have $\lim_{t \to \infty} g_i(a_i(t)) = 1$.
    Thus, there is exactly one point $a'_i \geq 0$ such that $g(a'_i)=1$, and we have $\lim_{t \to \infty} a_i(t) = a'_i$. 
    Hence, $\bw_i(\infty)$ and $\bv_i(\infty)$ exist.
    Moreover, $\bw_i(\infty) = \bw'_i + a'_i \bx_i \in \mathcal{S}_i \setminus \partial \mathcal{S}_i$.
\end{proof}

\begin{proof}[Proof of \thmref{T3}]
    We denote
    \begin{align*}
        \mathcal{W} = &\left\{ W:~
            \norm{\bw_i} \in \left( \frac{\sqrt{3}}{2}, \sqrt{ \frac{1}{4} + \frac{4}{3 \left(\cos  \max{ \set{ \arcsin{\frac{2 \norm{\bw_i(0)}}{\sqrt{3}}}, \measuredangle\big(\bx_1, \bx_2\big) - \frac{\pi}{2} } } \right)^2 } } \right)
            \; \forall i \in [2], 
            \right.
            \\
            &
            \;\;\;\;\;\;\;\;\;\;\;
            \text{ and }\measuredangle \left( \bw_1, \bw_2 \right) \in \left[ \pi - \measuredangle(\bx_1, \bx_2), \measuredangle(\bx_1, \bx_2) + 2 \arcsin{\frac{2 \max_{i \in [2]} \norm{\bw_i(0)}}{\sqrt{3}}} \right]
        \Bigg\}~.
    \end{align*}

    By \lemref{lem:GF converges} if we initialize $\bv_i(0)=\zero$ and $\bw_i(0) \in \mathcal{S}_i \setminus \partial \mathcal{S}_i$ for all $i \in [2]$, then GF converges and we have $L_{X,Y}(\btheta(\infty)) = 0$ and $\bw_i(\infty) \in \mathcal{S}_i \setminus \partial \mathcal{S}_i$ for all $i \in [2]$. 
    Also, by our assumption we have 
    \[
        \norm{\bw_i(0)} 
        \leq \frac{\sqrt{3}}{2}\sin\left(\frac{\pi - \measuredangle(\bx_1,\bx_2)}{4}\right)
        < \frac{\sqrt{3}}{2}\sin\left(\frac{\pi}{8}\right)
        < \frac{1}{2}~.
    \]
    Therefore, by \lemref{lemma:ws_bounded_norm_and_angle}, $W(\infty) \in \cw$. From the same arguments, $W(\infty) \in \cw$ also if the initialization of $\bw_i$ is such that $\bw_i(0) \in \mathcal{S}_{3-i} \setminus \partial \mathcal{S}_{3-i}$ for all $i \in [2]$. 
    Hence, 
    \begin{align} \label{eq:pr cw}
        \Pr&\left[W(\infty) \in \mathcal{W} \text{ and } L_{X,Y}(\btheta(\infty))=0 \right] \nonumber
        \\
        &\geq
        \Pr\left[ 
            \bw_i(0) \in \mathcal{S}_i \setminus \partial \mathcal{S}_i \; \forall i \in [2] 
            \text{ or } \bw_i(0) \in \mathcal{S}_{3-i} \setminus \partial \mathcal{S}_{3-i} \; \forall i \in [2]
        \right] \nonumber
        \\
        &=2 \cdot \Pr\left[ 
            \bw_i(0) \in \mathcal{S}_i \setminus \partial \mathcal{S}_i \; \forall i \in [2] 
        \right] \nonumber
        \\
        &= 2 \cdot \frac{\alpha(\mathcal{S}_1)}{2\pi} \cdot \frac{\alpha(\mathcal{S}_2)}{2\pi}~,
    \end{align}
    where $\alpha(\mathcal{S}_i)$ is the angle that corresponds to the region $\mathcal{S}_i$. Formally, the angle of a region $\cs_i$ is defined by $\alpha(\cs_i) = \measuredangle(\ba_1,\ba_2)$ where $\ba_1,\ba_2 \in \partial \cs_i$ are linearly~independent.
    
    Let $\bs_i \in (\partial \cs_i) \cap (\partial \cs)$ and let $\bd_i \in (\partial \cs_i) \cap (\partial \cd)$. Note that $\measuredangle(\bs_i,\bx_i) = \measuredangle(\bx_1,\bx_2) - \frac{\pi}{2}$ and that $\measuredangle(\bd_i,\bx_i) = \frac{\pi}{2}$. Thus,
    \[
        \alpha(\cs_i)
        = \measuredangle(\bs_i,\bd_i)
        = \measuredangle(\bs_i,\bx_i) +\measuredangle(\bd_i,\bx_i)
        = \measuredangle(\bx_1,\bx_2) - \frac{\pi}{2} + \frac{\pi}{2}
        = \measuredangle(\bx_1,\bx_2)~.
    \]
    Combining the above with \eqref{eq:pr cw} we get 
    \[
        \Pr\left[W(\infty) \in \mathcal{W} \text{ and } L_{X,Y}(\btheta(\infty))=0 \right]
        \geq 2 \cdot \left(\frac{\measuredangle(\bx_1,\bx_2)}{2\pi} \right)^2~.
    \]
    
    Finally, since 
    \[
        \norm{\bw_i(0)} 
        \leq \min\left\{
            \frac{\sqrt{3}}{2} \sin\left(\frac{\pi - \measuredangle(\bx_1,\bx_2)}{4} \right),
            \frac{\sqrt{3}}{2} \sin \left(\measuredangle(\bx_1,\bx_2) - \frac{\pi}{2}\right)
        \right\}~,
    \]
    then $W(\infty) \in \mathcal{W}$ implies that for all $i \in [2]$ we have
    \begin{align*}
        \norm{\bw_i(\infty)} 
        &\in
        \left( \frac{\sqrt{3}}{2}, \sqrt{ \frac{1}{4} + \frac{4}{3 \left(\cos \max{ \set{ \arcsin{\frac{2 \cdot \frac{\sqrt{3}}{2} \sin \left(\measuredangle(\bx_1,\bx_2) - \frac{\pi}{2}\right)}{\sqrt{3}}}, \measuredangle\big(\bx_1, \bx_2\big) - \frac{\pi}{2} } } \right)^2 } } \right)
        \\
        &= \left( \frac{\sqrt{3}}{2}, \sqrt{ \frac{1}{4} + \frac{4}{3 \left(\cos (\measuredangle\big(\bx_1, \bx_2\big) - \frac{\pi}{2}) \right)^2 } } \right)
        \\
        &= \left( \frac{\sqrt{3}}{2}, \sqrt{ \frac{1}{4} + \frac{4}{3 \left(\sin \measuredangle\big(\bx_1, \bx_2\big) \right)^2 } } \right)~,
    \end{align*}
    and
    \begin{align*}
        \measuredangle \left( \bw_1(\infty), \bw_2(\infty) \right) 
        &\in \left[ \pi - \measuredangle(\bx_1, \bx_2), \measuredangle(\bx_1, \bx_2) + 2 \arcsin{\frac{2 \cdot \frac{\sqrt{3}}{2} \sin\left(\frac{\pi - \measuredangle(\bx_1,\bx_2)}{4} \right) }{\sqrt{3}}} \right]
        \\
        &= \left[ \pi - \measuredangle(\bx_1, \bx_2), \measuredangle(\bx_1, \bx_2) + \frac{\pi - \measuredangle(\bx_1, \bx_2)}{2} \right]~. 
    \end{align*}
\end{proof}

\section{Proof of \thmref{thm:positive-square-loss}}

Let $\alpha = \left( \frac{1}{B} \right)^{\frac{k'-k}{k'}}$.
Consider the following fully-connected~network~$N'$ of width~$m$ and depth~$k'$. The weight~matrices of layers $i \in [k]$ in $N'$ are $W'_i = \alpha W_i$.
Note that the $k$-th layer in $N'$ contains a single neuron, and that since the weights in the first $k$ layers of $N'$ are obtained from the weights of $N$ by scaling with the parameter $\alpha$, then for every input $\bx_i$ in the dataset the input to the neuron in layer $k$ in $N'$ is $\alpha^k \cdot N(\bx_i) = \alpha^k y_i \geq 0$.
The layers $i \in \{k+1,\ldots,k'\}$ in $N'$ are of width~$1$. Hence, their weight~matrices are of dimension $1 \times 1$. We define these weights by $W'_i = \beta$ for $\beta := \left( \frac{1}{B} \right)^{-\frac{k}{k'}}$.
Thus, for an input $\bx_i$ we have 
\[
    N'(\bx_i) 
    = \alpha^k y_i \cdot \beta^{k'-k}
    = y_i \left( \frac{1}{B} \right)^{\frac{k'-k}{k'} \cdot k} \left( \frac{1}{B} \right)^{-\frac{k}{k'}\cdot (k'-k)}
    = y_i~.
\]

Let $\btheta' = \left[W'_1,\ldots,W'_{k'} \right]$ be the parameters of $N'$. Let $N^*:=N_{\btheta^*}$ be the network with the parameters $\btheta^*$ that achieves a global optimum of Problem~\ref{eq:optimization problem square loss}.
Since the network~$N'$ is of depth~$k'$ and width~$m \leq m'$ and since the network~$N^*$ is a global optimum, then we have $\norm{\btheta^*} \leq \norm{\btheta}$. 
Therefore,
\begin{align} \label{eq:bound theta star square loss}
    \norm{\btheta^*}^2
    &\leq \norm{\btheta'}^2 \nonumber
    \\
    &= \left(\sum_{i=1}^{k} \alpha^2 \norm{W_i}_F^2 \right) + (k'-k) \beta^2 \nonumber
    \\
    &\leq \alpha^2 B^2 k + \beta^2 (k'-k) \nonumber
    \\
    &= \left( \frac{1}{B^2} \right)^{\frac{k'-k}{k'}} B^2 k + \left( \frac{1}{B^2} \right)^{-\frac{k}{k'}} (k'-k) \nonumber
    \\
    &= \left( \frac{1}{B^2} \right)^{-\frac{k}{k'}} k'~.
\end{align}

In the following lemma, we show that since $N^*$ is a global optimum of \eqref{eq:optimization problem square loss}, then its layers must be balanced:
\begin{lemma} \label{lem:optimal balanced}
    For every $1 \leq i < j \leq k'$ we have $\norm{W^*_i}_F = \norm{W^*_j}_F$.
\end{lemma}
\begin{proof}
    Let $1 \leq i < j \leq k'$. 
    For $\gamma>0$ we define a network~$N_\gamma$ which is obtained from $N^*$ as follows. 
    The network~$N_\gamma$ is obtained by multiplying the weight~matrix $W^*_i$ by $\gamma$, and the weight~matrix $W^*_j$ by $1/\gamma$. 
    Note that for every input $\bx$ we have $N_\gamma(\bx) = N^*(\bx)$.
    
    We have
    \[
        \frac{d}{d \gamma} \left(\norm{\gamma W^*_i}_F^2 + \norm{\frac{1}{\gamma}W^*_j}_F^2 \right)
         = 2 \gamma \norm{W^*_i}_F^2 - \frac{2}{\gamma^3} \norm{W^*_j}_F^2~.
    \]
    When $\gamma=1$ the above expression equals $2 \norm{W^*_i}_F^2 - 2 \norm{W^*_j}_F^2$. Hence, if $\norm{W^*_i}_F \neq \norm{W^*_j}_F$ then the derivative at $\gamma=1$ is non-zero, in~contradiction to the optimality of $N^*$.
\end{proof}

By the above lemma, there is $B^*>0$ such that 
$B^* =  \norm{W^*_i}_F$ 
for all 
$i \in [k']$.
By \eqref{eq:bound theta star square loss} we have
\begin{align*}
    (B^*)^2 \cdot k'
    = \norm{\btheta^*}^2
    \leq \left( \frac{1}{B^2} \right)^{-\frac{k}{k'}} k'~.
\end{align*}
Hence, for every $i \in [k']$ we have 
\begin{equation} \label{eq:bound frobenius square loss}
    \norm{W^*_i}_F^2
    = (B^*)^2 
    \leq \left( \frac{1}{B^2} \right)^{-\frac{k}{k'}}~.
\end{equation}

Moreover, since there is $i \in [n]$ with $\norm{\bx_i} \leq 1$ and $y_i \geq 1$, then the network~$N^*$ satisfies
\[
    1
    \leq y_i 
    = N^*(\bx_i)
    \leq \norm{\bx_i} \prod_{i \in [k']} \norm{W^*_i}_\sigma
    \leq \prod_{i \in [k']} \norm{W^*_i}_\sigma
    \leq \left( \frac{1}{k'} \sum_{i \in [k']} \norm{W^*_i}_\sigma \right)^{k'}~,
\]
where the last inequality follows from the AM-GM inequality. Therefore, we have 
\[
    \frac{1}{k'} \sum_{i \in [k']} \norm{W^*_i}_\sigma \geq 1~.
\]
Combining the above with \eqref{eq:bound frobenius square loss} we get
\begin{align*}
    \frac{1}{k'} \sum_{i \in [k']} \frac{\norm{W^*_i}_\sigma}{\norm{W^*_i}_F}
    = \frac{1}{B^*} \cdot \frac{1}{k'} \sum_{i \in [k']} \norm{W^*_i}_\sigma
    \geq \left( \frac{1}{B} \right)^{\frac{k}{k'}}~.
\end{align*}

\qed

\section{Proof of \thmref{thm:positive-exp}}

Let $\alpha = \left( \frac{\sqrt{2}}{B} \right)^{\frac{k'-k}{k'}}$.
Consider the following fully-connected~network~$N'$ of width~$m$ and depth~$k'$. The weight~matrices of layers $i \in [k-1]$ in $N'$ are $W'_i = \alpha W_i$.
Let $\bu$ be the weight~vector of the output~neuron in~$N$. The $k$-th layer in $N'$ is defined by the weight~matrix $W'_k = \alpha \cdot \begin{bmatrix} \bu^\top \\ - \bu^\top \end{bmatrix}$. That is, the $k$-th layer in $N'$ has two neurons: the first neuron corresponds to the output neuron of $N$, and the second neuron to its negation. Note that since the weights in $N'$ are obtained from the weights of $N$ by scaling with the parameter $\alpha$, then for every input $\bx$ the input to the first neuron in layer $k$ in $N'$ is $\alpha^k \cdot N(\bx)$, and the input to the second neuron in layer $k$ is $-\alpha^k \cdot N(\bx)$.
The layers $i \in \{k+1,\ldots,k'-1\}$ in $N'$ are defined by the weight~matrices $W'_i = \beta I_2$, where $\beta := \left( \frac{\sqrt{2}}{B} \right)^{-\frac{k}{k'}}$ and $I_2$ is the identity matrix of dimension $2$. Finally, the $k'$-th layer in $N'$ is defined by the weight~vector~$\beta \cdot \begin{pmatrix} 1 \\ - 1 \end{pmatrix}$.
Note that given an input $\bx$, the first $k$ layers in $N'$ compute $\begin{pmatrix} \sigma\left(\alpha^k \cdot N(\bx)\right) \\ \sigma\left(- \alpha^k \cdot N(\bx) \right) \end{pmatrix}$, then the next $k'-k-1$ layers compute $\begin{pmatrix} \beta^{k'-k-1} \sigma\left(\alpha^k \cdot N(\bx)\right) \\ \beta^{k'-k-1} \sigma\left(- \alpha^k \cdot N(\bx) \right) \end{pmatrix}$, and finally the last layer returns 
\begin{align*}
    \beta^{k'-k} \sigma\left(\alpha^k \cdot N(\bx)\right) - \beta^{k'-k} \sigma\left(- \alpha^k \cdot N(\bx) \right) 
    &= \beta^{k'-k} \alpha^k \cdot N(\bx)
    \\
    &= \left( \frac{\sqrt{2}}{B} \right)^{-\frac{k}{k'} \cdot (k'-k)} \cdot \left( \frac{\sqrt{2}}{B} \right)^{\frac{k'-k}{k'} \cdot k} \cdot N(\bx)
    \\
    &= N(\bx)~.
\end{align*}
Thus, $N'(\bx)=N(\bx)$. 

Let $\btheta' = \left[W'_1,\ldots,W'_{k'} \right]$ be the parameters of $N'$. Let $N^*:=N_{\btheta^*}$ be the network with the parameters $\btheta^*$ that achieves a global optimum of Problem~\ref{eq:optimization problem}.
Since the network~$N'$ is of depth~$k'$ and width~$m \leq m'$ and since the network~$N^*$ is a global optimum, then we have $\norm{\btheta^*} \leq \norm{\btheta}$. Therefore,
\begin{align} \label{eq:bound theta star}
    \norm{\btheta^*}^2
    &\leq \norm{\btheta'}^2 \nonumber
    \\
    &= \left(\sum_{i=1}^{k-1} \alpha^2 \norm{W_i}_F^2 \right)  
    + \alpha^2 \left( 2 \norm{W_k}_F^2 \right)  
    + (k'-k-1) \beta^2 \cdot 2 + \beta^2 \cdot 2 \nonumber
    \\
    &\leq \alpha^2 (k-1) B^2 
    + \alpha^2 \cdot 2B^2 + \left( 2(k'-k-1)+2\right)\beta^2 \nonumber
    \\
    &= \alpha^2 B^2 (k+1) + \beta^2 \cdot 2(k'-k)  \nonumber
    \\
    &= \left( \frac{2}{B^2} \right)^{\frac{k'-k}{k'}} B^2 (k+1) + \left( \frac{2}{B^2} \right)^{-\frac{k}{k'}} \cdot 2(k'-k) \nonumber
    \\
    &= 2 \cdot \left( \frac{2}{B^2} \right)^{-\frac{k}{k'}} (k+1) + \left( \frac{2}{B^2} \right)^{-\frac{k}{k'}} \cdot 2 (k'-k) \nonumber
    \\
    &= 2 \cdot \left( \frac{2}{B^2} \right)^{-\frac{k}{k'}} (k'+1)~.
\end{align}

The following lemma shows that since $N^*$ is a global optimum of \eqref{eq:optimization problem}, then its layers must be balanced:
\begin{lemma}
    For every $1 \leq i < j \leq k'$ we have $\norm{W^*_i}_F = \norm{W^*_j}_F$.
\end{lemma}

The proof of the lemma is similar to the proof of \lemref{lem:optimal balanced}.
By the lemma, there is $B^*>0$ such that $B^* =  \norm{W^*_i}_F$ for all $i \in [k']$.
By \eqref{eq:bound theta star} we have
\begin{align*}
    (B^*)^2 \cdot k'
    = \norm{\btheta^*}^2
    \leq 2 \cdot \left( \frac{2}{B^2} \right)^{-\frac{k}{k'}} (k'+1)~.
\end{align*}
Hence, for every $i \in [k']$ we have 
\begin{equation} \label{eq:bound frobenius}
    \norm{W^*_i}_F^2
    = (B^*)^2 
    \leq 2 \cdot \left( \frac{2}{B^2} \right)^{-\frac{k}{k'}} \cdot \frac{k'+1}{k'}~.
\end{equation}

Moreover, since there is $i \in [n]$ with $\norm{\bx_i} \leq 1$ and $|y_i|=1$, then the network~$N^*$ satisfies
\[
    1
    \leq y_i N^*(\bx_i)
    \leq \left| N^*(\bx_i) \right|
    \leq \norm{\bx_i} \prod_{i \in [k']} \norm{W^*_i}_\sigma
    \leq \prod_{i \in [k']} \norm{W^*_i}_\sigma
    \leq \left( \frac{1}{k'} \sum_{i \in [k']} \norm{W^*_i}_\sigma \right)^{k'}~,
\]
where the last inequality follows from the AM-GM inequality. Therefore, we have \[
    \frac{1}{k'} \sum_{i \in [k']} \norm{W^*_i}_\sigma \geq 1~.
\]
Combining the above with \eqref{eq:bound frobenius} we get
\begin{align*}
    \frac{1}{k'} \sum_{i \in [k']} \frac{\norm{W^*_i}_\sigma}{\norm{W^*_i}_F}
    &= \frac{1}{B^*} \cdot \frac{1}{k'} \sum_{i \in [k']} \norm{W^*_i}_\sigma
    \geq \left( 2 \cdot \left( \frac{2}{B^2} \right)^{-\frac{k}{k'}} \cdot \frac{k'+1}{k'} \right)^{-1/2} \cdot 1
    \\
    &= \frac{1}{\sqrt{2}} \cdot \left( \frac{\sqrt{2}}{B} \right)^{\frac{k}{k'}} \cdot \sqrt{\frac{k'}{k'+1}}~.
\end{align*}

\qed

\end{document}